\documentclass[11pt]{article}

\usepackage[margin=1in]{geometry}
\usepackage[T1]{fontenc}
\usepackage{lmodern}
\usepackage{microtype}
\usepackage{amsmath,amssymb,amsthm,physics}
\usepackage{mathtools}
\usepackage{graphicx}
\usepackage{booktabs}
\usepackage{hyperref}
\usepackage[nameinlink]{cleveref}
\usepackage{algorithm}
\usepackage{algpseudocode}

\usepackage{subcaption} 
\theoremstyle{plain} 

\newtheorem{theorem}{Theorem}
\newtheorem{lemma}[theorem]{Lemma}
\newtheorem{definition}[theorem]{Definition}
\newtheorem{corollary}[theorem]{Corollary}
\newtheorem{remark}[theorem]{Remark}

\usepackage{natbib}
\hypersetup{
  colorlinks=true,
  linkcolor=blue,
  citecolor=blue,
  urlcolor=blue
}
\DeclareMathOperator*{\argmin}{arg\,min}
\title{On the Loss Landscape Geometry of Regularized Deep Matrix Factorization: Uniqueness and Sharpness}
\author{Anıl Kamber \\ UC San Diego \\ \texttt{akamber@ucsd.edu} \and Rahul Parhi \\ UC San Diego \\ \texttt{rahul@ucsd.edu}}
\date{}

\begin{document}

\maketitle
\begin{abstract}
    Weight decay is ubiquitous in training deep neural network architectures. Its empirical success is often attributed to capacity control; nonetheless, our theoretical understanding of its effect on the loss landscape and the set of minimizers remains limited. In this paper, we show that $\ell^2$-regularized deep matrix factorization/deep linear network training problems with squared-error loss admit a unique end-to-end minimizer for all target matrices subject to factorization, except for a set of Lebesgue measure zero formed by the depth and the regularization parameter. This observation reveals fundamental properties of the loss landscape of regularized deep matrix factorization problems: \emph{the Hessian spectrum is constant across all minimizers} of the regularized deep scalar factorization problem with squared-error loss. Moreover, we show that, in regularized deep matrix factorization problems with squared-error loss, if the target matrix does not belong to the Lebesgue measure-zero set, then the Frobenius norm of each layer is constant across all minimizers. This, in turn, yields a global lower bound on the trace of the Hessian evaluated at any minimizer of the regularized deep matrix factorization problem. Furthermore, we establish a critical threshold for the regularization parameter above which the unique end-to-end minimizer collapses to zero.
\end{abstract}
\textbf{Keywords:} $\ell^2$ regularization, weight decay, deep matrix factorization, deep linear networks, flatness, threshold, Schatten norm, gradient descent

\begingroup

\renewcommand\thefootnote{}
\footnotetext{Copyright © 2026 by Anıl Kamber and Rahul Parhi}
\endgroup


\section{Introduction}

Weight decay/$\ell^2$ regularization is an \emph{explicit regularization}\footnote{Explicit regularization is typically implemented by incorporating into the loss function a complexity penalty scaled by a regularization parameter, and the level of regularization imposed on the optimization problem can be adjusted by tuning this parameter. The state-of-the-art implementation is decoupled weight decay \citep{loshchilov2017decoupled} with the Adam optimizer.} technique to improve the performance of deep neural networks. \citet{plaut1986experiments} were the first to suggest shrinking network weights during training, and \citet{hinton1987learning} later observed that weight decay improved the performance of a deep neural network trained for a shape recognition task by $76$\%. Today, it is widely used \citep{krizhevsky2012imagenet, simonyan2014very,he2016deep,devlin2019bert,brown2020language}. Despite its widespread use, our theoretical understanding of why explicit $\ell^2$ regularization of weights during training leads to models that generalize better than those trained without it remains limited. 

The optimization landscapes of deep learning architectures are highly nonconvex, even though the loss functions themselves are convex. Therefore, internal representations of the network are learned by the gradient-based optimization algorithms. However, due to the nonconvex structure of the loss landscape, gradient-based algorithms can become trapped in spurious minima or nonstrict saddles. Moreover, if they are able to escape these critical points, overparameterized models can represent infinitely many functions that perfectly fit the training data, so how do gradient-based algorithms avoid bad minima? Considering the empirical success of $\ell^2$ regularization, it is natural to hypothesize that $\ell^2$ regularization reshapes the loss landscape and the set of minima in a way that the effect of these problems on the learning dynamics is reduced.

To this end, \citet{chen2025complete} showed that under a necessary and sufficient condition on the regularization parameters, the $\ell^2$-regularized deep matrix factorization problem is partially benign. This means that every critical point is either a local minimum or a strict saddle from which gradient-based optimization algorithms escape almost always. \citet{liang2025stable} empirically observed that explicit regularization seems to break the edge-of-stability phenomenon. Recently, \citet{boursier2025benignitylosslandscapeweight} presented a probabilistic analysis of the loss landscape of $\ell^2$-regularized two-layer ReLU networks. They showed that when the network is sufficiently overparameterized, almost all partitioning cones of the parameter space contain no spurious minima. However, this does not necessarily mean that almost all local minima are global minima. Supporting this argument, they showed that this benignity holds relevance under large initialization. 

In this paper, we investigate the structure of the solution set of $\ell^2$-regularized deep matrix factorization/deep linear network training problems. Since generalization is fundamentally a property of the function space rather than the parameter space, we analyze how explicit $\ell^2$ regularization reshapes the set of functions represented by global minima. It is well known that deep linear networks with quadratic loss admit a unique end-to-end minimizing function determined by the second-order statistics of the training data \citep{mulayoff}. Motivated by this observation, we ask the fundamental question of whether the minimizing end-to-end function remains unique once $\ell^2$ regularization is introduced. To this end, we investigate the $\ell^2$-regularized deep matrix factorization problem with squared-error loss by highlighting the fact that deep matrix factorization and deep linear neural network training problems are equivalent when the input data covariance matrix is full rank \citep{chou2024gradient}. Note that this investigation is also quite remarkable for the regularized nonlinear networks since, in particular, ReLU activation partitions the parameter space into cones. Inside each of these cones, the network acts as a linear model. Hence, to answer the question of how rich the set of functions represented by global minimizers within each cone is, this investigation is a good start.

We present our main results in \Cref{sec:mainresults}. In \Cref{sec:discussion}, we discuss how we bring the notions from low-rank matrix recovery to analyze the geometry of the loss landscape and the structure of the set of minimizers of the $\ell^2$-regularized deep matrix factorization problem. We conclude in \Cref{sec:conclusion}.

\subsection{Contributions}

In this paper, we characterize various remarkable aspects of the geometry of the loss landscape near minima and the structure of the set of minima in $\ell^2$-regularized deep matrix factorization problems with squared-error loss. At a glance, our contributions are as follows:

\begin{itemize}
    \item $\ell^2$-regularized deep scalar factorization problems with squared-error loss admit a unique end-to-end minimizer for all scalars except two (\Cref{theorem:uniqueminscalar}). Furthermore, we characterize the full Hessian spectrum across all minimizers of the regularized deep scalar factorization problem for all scalars except these two. We observe that the Hessian spectrum is constant across all minimizers, and the maximum Hessian eigenvalue depends on the depth, the magnitude of optimal layers, and the regularization parameter. This implies that under Hessian-based sharpness measures, \emph{all global minima are equally flat almost always}. To the best of our knowledge, our results offer the first complete characterization of Hessian spectrum across minimizers in deep-factorization-type problems (\Cref{theorem:hessianspectrumisconstant}). 

    \item The singular vectors of the end-to-end product of any minimizer of the $\ell^2$-regularized deep matrix factorization problem must align with those of the target matrix (\Cref{theorem:maintheorem}).

    \item We show that $\ell^2$-regularized deep matrix factorization problems with squared-error loss admit a unique end-to-end minimizer for all target matrices subject to factorization, except for a set of Lebesgue measure zero formed by the depth and the regularization parameter (\Cref{theorem:maintheorem}). This implies that the regularized deep matrix factorization problem admits a unique end-to-end minimizer \emph{almost always}.

    \item \citet{chen2025complete} showed that, at any minimizer of the regularized deep matrix factorization problem, the layers are Frobenius-norm balanced. We extend this result as follows: we show that if the target matrix does not belong to the Lebesgue measure-zero set, then the Frobenius norm of each layer is the same across all minimizers (\Cref{corollary:extended}).

    \item We establish a critical threshold for the regularization parameter above which the unique end-to-end minimizer collapses to zero (\Cref{corollary:collapse}).
    
    \item We present a lower bound for the trace of the Hessian matrix evaluated at any minimizer of the regularized deep matrix factorization problem that holds relevance if the target matrix does not belong to the Lebesgue measure-zero set. We leave the question of whether this lower bound is achieved at a flat minimum as an open problem (\Cref{theorem:trace}).

\end{itemize}
In general, theoretical analyses of deep matrix factorization or deep neural network training problems focus either on global properties of the loss landscape and the set of minima \citep{kawaguchi2016deep,ge2017no,laurent2018deep, zhou2022optimization,ethpaper,boursier2025benignitylosslandscapeweight,kim2025exploringlosslandscaperegularized,josz2025geometryflatminima,kamber2026sharpnessminimadeepmatrix} or on the dynamics of gradient-based optimization methods \citep{gunasekar2017implicit,chizat2018global,implicitregularizationindeepmatrixfactorization, chou2024gradient,liang2025gradientdescentlargestep,boursier2025simplicity,ghosh2025learning}. This paper falls into the former category, as we analyze how $\ell^2$ regularization reshapes the loss landscape and the set of minima.

\subsection{Related Work}

\paragraph{Benign and Partially Benign Landscapes}
A nonconvex loss landscape can exhibit several properties that impede gradient-based optimization. Characterization of these malign properties can help the practitioner choose a better training strategy, while their absence can help explain the empirical success of gradient-based methods in neural network training. To this end, \citet{chen2025complete} defined a loss landscape as  \emph{benign} when every local minimum is global, and every saddle point is a strict saddle point\footnote{A strict saddle point is defined as a critical point where the Hessian has at least one strictly negative eigenvalue. \citep{ge2015escaping,lee2016gradientescaping} showed that gradient descent (GD) can escape from strict saddle points.}. Furthermore, they defined a loss landscape as \emph{partially benign} when either every local minimum is global or every saddle point is a strict saddle point. We adopt these definitions throughout the paper. Another malign property characterized in the literature is the presence of \emph{spurious valleys}. Spurious valleys\footnote{\citet{venturi2019spurious} defined a spurious valley as one connected component of the set $\{\theta : \mathcal{L}(\theta)\le c\}$ that does not contain global minima, where $\mathcal{L}$ is the training loss.} are defined as connected components of a sub-level set of training loss that do not contain global minima~\citep{venturi2019spurious}. Note that every spurious valley contains a local minimum, whereas the converse---every local minimum lies within a spurious valley---is not necessarily true \citep[Figure~8c]{le2023spurious}. Moreover, \citet{liang2022revisiting} showed that there exist paths in the parameter space of a deep linear network training problem along which the loss decreases and diverges to infinity. They named these paths as \emph{decreasing paths to infinity}, and argued that a desirable landscape should be free of both spurious minima and these decreasing paths.

\paragraph{Loss Landscapes of Deep Networks} Several works showed that the landscape of deep matrix factorization/deep linear network training problem with squared-error loss is partially benign \citep{laurent2018deep,sun2020global}. Most notably, \citet{kawaguchi2016deep} showed that the loss landscape of the deep linear network training problem with squared-error loss is benign when the number of layers is at most three. For depths greater than three, however, the landscape becomes partially benign: every local minimum is global, but non-strict saddle points exist. Moreover, \citet{hardt2016identity} showed that deep linear residual networks have a partially benign loss landscape, i.e., every local minimum is global. 

On the other hand, for regularized deep networks, \citet{haeffele2017global} presented sufficient conditions that guarantee global optimality of every local minimum. Most notably, \citet{chen2025complete} presented a closed-form characterization of all critical points of the $\ell^2$-regularized deep matrix factorization problem. Furthermore, they provided precise conditions under which a critical point of the regularized deep matrix factorization problem is either a global minimum, a local minimum, a strict saddle, or a non-strict saddle. Lastly, they derived a necessary and sufficient condition on the regularization parameters under which the regularized deep matrix factorization problem is partially benign, i.e., every critical point is either a local minimizer or a strict saddle point. Recently, \citet{boursier2025benignitylosslandscapeweight} presented a probabilistic analysis of the loss landscape of the $\ell^2$-regularized two-layer ReLU networks that extends the work of \citet{karhadkar2023mildly} to the regularized setting. They showed that when the network is sufficiently overparameterized, almost all partitioning cones of the parameter space contain no spurious minima.

\paragraph{Structure of the Set of Global Minima}
\citet{garipov2018loss} empirically observed that optimal points in the parameter space are connected by curves along which the training and test losses remain approximately constant. \citet{simsek2021geometry} showed that for a multi-layer perceptron, adding a single neuron to each layer connects discrete symmetry-induced global minima into a single manifold. \citet{tolgamert} showed that a neural network training problem can be reformulated as a convex optimization problem. Furthermore, several studies have shown that various neural network training problems with $\ell^2$ regularization/weight decay also have such convex reformulations \citep{sahiner2020vector,ergen2020training,ergen2021global}. Leveraging this convex characterization, \citet{kimmert2024exploring} revealed several properties of the loss landscape of regularized neural networks by reformulating the training objective as an equivalent convex optimization problem and considering its dual. In particular, they characterized the solution set of the convex reformulation of the two-layer neural network training problem, and showed that the set of global minima undergoes a phase transition as the network gets wider.

\section{Notation, Preliminaries, and Problem Setup}

We use the following notation throughout the paper. We denote by $\otimes$ the \emph{Kronecker product} and by $\langle \cdot , \cdot \rangle$ the \emph{Frobenius inner product}. We denote by $\sigma_{\max}(\cdot)$ the \emph{spectral norm} and by $\norm{\cdot}_F$ the \emph{Frobenius norm}. For $p \in [0,\infty]$, we denote by $\norm{\cdot}_p$ the $\ell^p$ norm/quasi-norm\footnote{$\|\cdot\|_p$ is a quasi-norm when $0 < p < 1$.} and by $\norm{\cdot}_{\mathcal{S}^p}$ the Schatten-$p$ norm/quasi-norm. Furthermore, we write $[L] := \{1,2,\dots,L\}$ for the set of natural numbers up to $L$, represent a matrix $\mathbf{X} \in \mathbb{R}^{m \times n}$ as $[x_{ij}] \in \mathbb{R}^{m \times n}$, denote by $\mathbf{1} := [1 \quad 1 \quad \cdots \quad 1] \in \mathbb{R}^n$ the vector whose entries are all equal to $1$, by $\mathbf{e}_i \in \mathbb{R}^n$ the $i$th standard basis vector, by $\mathbb{R}_+^n$ the non-negative orthant of $\mathbb{R}^n$, and by ${\mathbb{R}^n}^\downarrow$ the non-increasingly ordered $\mathbb{R}^n$. We denote by $(a,b)$ the open interval between $a$ and $b$, where $a < b$. Lastly, we denote by $A^C$ the complement of a set $A$.

We use the following definition throughout the paper for the singular value decomposition.
\begin{definition}
Let  $\mathbf{X} \in \mathbb{R}^{m \times n}$ and $r:= \min \{m,n\}$. We define the full singular value decomposition of $\mathbf{X}$ as
\begin{equation}
\mathbf{X} = \mathbf{U}_\mathbf{X} \boldsymbol{\Sigma}_\mathbf{X} \mathbf{V}_{\mathbf{X}}^\top,
\end{equation}
where $\mathbf{U}_\mathbf{X} \in \mathbb{R}^{m \times m}$ and $\mathbf{V}_\mathbf{X} \in \mathbb{R}^{n \times n}$ are orthogonal matrices, and $\mathbf{\Sigma}_\mathbf{X} \in \mathbb{R}^{m \times n}$ is a diagonal matrix. Furthermore, we denote the vector of non-increasingly ordered singular values of a matrix $\mathbf{X} \in {\mathbb{R}^{m \times n}}$ by $\sigma({\mathbf{X}}) \in {\mathbb{R}^{\min \{m,n\}}}^\downarrow$. Without loss of generality, we assume that the diagonal entries in $\mathbf{\Sigma}_\mathbf{X}$ are sorted non-increasingly.
\end{definition}

\subsection{Problem Setup}
We consider the following objective function
\begin{equation}
    \min_{\mathbf{w}\in \mathbb{R}^N}\mathcal{L}(\mathbf{w}) := \norm {\mathbf{M}^\natural-\mathbf{W}_L \mathbf{W}_{L-1}\cdots \mathbf{W}_1}_F^2 +\frac{\lambda}{L}\sum_{i=1}^{L} \norm{\mathbf{W}_i}_F^2,
    \label{mainobjectivefunction}
\end{equation}
where $\mathbf{M}^\natural \in \mathbb{R}^{d_L \times d_0}$ is the matrix of interest, $L \geq 2$ is the depth, $\mathbf{W}_i \in \mathbb{R}^{d_i \times d_{i-1}}$ is the $i^{th} $ factor (layer), $\lambda > 0$ is the regularization parameter, and $\mathbf{w} = \operatorname{vec}(\mathbf{W}_1,\ldots,\mathbf{W}_L)$ denotes the full set of parameters. To guarantee the feasibility of the factorization of every point in $\mathbb{R}^{d_L \times d_0}$, we require $\min_{i} d_i \geq \min \{d_0, d_L \}$. This form of explicit regularization is well known to be equivalent to Schatten-$(2/L)$ regularization~\citep{schatten1,schatten2}. Define the solution set of the optimization criterion as
$\Omega := \argmin_{\mathbf{w} \in \mathbb{R}^{N}} \mathcal{L}(\mathbf{w})$, where $N := \sum_{i=1}^{L} d_i d_{i-1}$ is the total number of parameters.  To simplify the notation for subsequent derivations, we define 
\begin{equation}
    \prod_{j = n}^ m \mathbf{W}_j := \begin{cases}
        \mathbf{W}_m \mathbf{W}_{m-1} \dots {\mathbf{W}_{n}} \quad &\text{if $n \leq m$}, \\
        \mathbf{I}_{d_m} &\text{o.w.}, \forall n,m \in [L],
    \end{cases}
\end{equation}
where $\mathbf{W}_m \in \mathbb{R}^{d_m \times d_{m-1}}$.

Recently, \citet[Lemma~3.2]{chen2025complete} showed that, at each minimizer of $\mathcal{L}$, every layer has exactly the same singular values. We can formulate this observation in the following theorem.
\begin{theorem}[{\cite[Lemma 3.2]{chen2025complete}}]
For any $\mathbf{w}^* \in \Omega$, layers (factors) are balanced, i.e.,
\begin{equation}
\mathbf{W}_i^*\,\mathbf{W}_i^{*\top}
= \mathbf{W}_{i+1}^{*\top}\,\mathbf{W}_{i+1}^* \quad \forall i \in [L-1],
\end{equation}
which implies that layers possess exactly the same singular values. Furthermore, if the singular values are distinct, then their left and right singular vectors align, up to an incurable sign ambiguity.
\label{chenetal}
\end{theorem}
\noindent
Furthermore, we can formulate an equivalent optimization problem to (\ref{mainobjectivefunction}) by using the variational form of the Schatten-$2/L$ quasi-norm (see \Cref{appendix:schattenvariational}, \Cref{theorem:variationalform}).
\begin{theorem}
    Consider the following optimization objective
    \begin{equation}
        \min_{\mathbf{M}\in \mathbb{R}^{d_L\times d_0}} \mathcal{L}_{\mathbf{M}^{\natural}}(\mathbf{M}) := \norm{\mathbf{M}^{\natural}-\mathbf M}_F^2 + \lambda\norm{\mathbf M}_{\mathcal{S}^{2/L}}^{2/L},
        \label{equiavlentform}
    \end{equation}
  where $\mathbf{M}^\natural\in \mathbb{R}^{d_L \times d_0}$ is the target matrix, $\lambda > 0$ is the regularization parameter, $L \geq 1$ denotes the depth. Also, consider the optimization objective $\mathcal{L}(\mathbf{w})$ defined in (\ref{mainobjectivefunction}). Denote by 
  \begin{equation}
      S := \argmin_{\mathbf{M}\in \mathbb{R}^{d_L \times d_0}}\mathcal{L}_{\mathbf{M}^{\natural}}(\mathbf{M}) \quad \text{and} \quad  \Omega:=\argmin_{\mathbf{w}\in \mathbb{R}^N} \mathcal{L}(\mathbf{w})
  \end{equation}
  the set of minimizers of the optimization problems in (\ref{equiavlentform}) and (\ref{mainobjectivefunction}), respectively. Then for any $\mathbf{w}^* \in \Omega$, $\prod_{i=1}^L \mathbf{W}_i^* \in S$. Furthermore, for any $\mathbf{M}^* \in S$, there exists $\mathbf{w}^* \in \Omega$ such that $\prod_{i=1}^L \mathbf{W}_i^* = \mathbf{M}^*$.
  \label{equivalentoptimization}
\end{theorem}

\begin{proof}
By using the variational form of the Schatten-$2/L$ quasi-norm (see \Cref{theorem:variationalform}), minimizing $\mathcal{L}(\mathbf{M})$ is equivalent to 
\begin{equation}
    \min_{\mathbf{M}\in \mathbb{R}^{d_L\times d_0}} \left \{\norm{\mathbf{M}^{\natural}-\mathbf M}_F^2 + \min_{\substack{\mathbf{W}_1, \cdots,\mathbf{W}_L: \\
        \mathbf{W}_L \cdots \mathbf{W}_1 = \mathbf{M}}} \frac{\lambda}{L} \sum_{i=1}^L \norm{\mathbf{W}_i}_F^2\right \}.
\end{equation}
This is equivalent to 
\begin{align}
    &\min_{\mathbf{M}\in \mathbb{R}^{d_L\times d_0}} \left \{ \min_{\substack{\mathbf{W}_1, \cdots,\mathbf{W}_L: \\
        \mathbf{W}_L \cdots \mathbf{W}_1 = \mathbf{M}}}\left \{\norm{\mathbf{M}^{\natural}-\mathbf M}_F^2 +  \frac{\lambda}{L} \sum_{i=1}^L \norm{\mathbf{W}_i}_F^2\right \} \right \} \\ 
        =&\min_{\mathbf{M}\in \mathbb{R}^{d_L\times d_0}} \left \{ \min_{\substack{\mathbf{W}_1, \cdots,\mathbf{W}_L: \\
        \mathbf{W}_L \cdots \mathbf{W}_1 = \mathbf{M}}}\left \{\norm{\mathbf{M}^{\natural}-\mathbf{W}_L \mathbf{W}_{L-1} \cdots \mathbf{W}_1}_F^2 +  \frac{\lambda}{L} \sum_{i=1}^L \norm{\mathbf{W}_i}_F^2\right \} \right \}.
\end{align}
The condition 
\begin{equation}
    \min_{i} d_i \geq \min \{d_0,d_L\} \quad  \forall i \in [L]
\end{equation}
guarantees the feasibility of factorization for all points in $\mathbb{R}^{d_L \times d_0}$. Therefore, we can rewrite the optimization objective as follows.
\begin{equation}
    \min_{\mathbf{w} \in \mathbb{R}^N} \left\{ \norm{\mathbf{M}^{\natural}-\mathbf{W}_L \mathbf{W}_{L-1} \cdots \mathbf{W}_1}_F^2 + \frac{\lambda}{L} \sum_{i=1}^L \norm{\mathbf{W}_i}_F^2 \right \}. 
\end{equation}
\end{proof}

\begin{remark}
This means that we can examine the set of minimizers of the Schatten-$2/L$ regularized problem in (\ref{equiavlentform}) to understand the structure of the set of functions represented by the global minimizers of the $\ell^2$-regularized deep matrix factorization problem in (\ref{mainobjectivefunction}).   
\end{remark}

\section{Uniqueness of the End-to-End Minimizer in Regularized Deep Matrix Factorization}
\label{sec:mainresults}
\subsection{Deep Scalar Factorization}
Before delving into the more general results, we first consider the problem of $\ell^2$-regularized deep scalar factorization. 
\begin{equation}
    \mathcal{L}(\mathbf{w}) := \underbrace{(m-w_L w_{L-1}\cdots w_1)^2}_{=:\, D(\mathbf{w})}+ \underbrace{\frac{\lambda}{L}\sum_{i=1}^{L} w_i^2}_{=:\, R(\mathbf{w})},
    \label{oopregsc}
\end{equation}
where $m \in \mathbb{R}$ denotes the scalar of interest, $L \geq 1$ is the depth, $w_i \in \mathbb{R}$ is the $i^{\text{th}}$ factor (layer), $\lambda > 0$ is the regularization parameter, and $\mathbf{w} = (w_1, \ldots, w_L) \in \mathbb{R}^L$ denotes the full set of parameters. Define the solution set of the optimization criterion as
$\Omega := \arg\min_{\mathbf{w} \in \mathbb{R}^{L}} \mathcal{L}(\mathbf{w})$. We observed that for any $m$ except two, $\mathcal{L}$ has a unique end-to-end minimizer.

\begin{theorem}
    Consider the objective function for the $\ell^2$-regularized deep scalar factorization problem in~(\ref{oopregsc}). Define $q := 2/L$. For any $m \in \mathbb{R} \setminus \left \{\pm\left (1-\frac{q}{2} \right)\lambda^{\frac{1}{2-q}}\left(1-q\right)^{\frac{q-1}{2-q}}  \right \}$, $\mathcal{L}$ has a unique end-to-end minimizer. Furthermore, the end-to-end minimizer is characterized as 
    \begin{equation}
\prod_{i=1}^L w^*_i = 
\begin{cases}
0
&|m| < \left (1-\frac{q}{2} \right)\lambda^{\frac{1}{2-q}}\left(1-q\right)^{\frac{q-1}{2-q}},\\
\rho^*(m)
&|m| > \left (1-\frac{q}{2} \right)\lambda^{\frac{1}{2-q}}\left(1-q\right)^{\frac{q-1}{2-q}},\\

\left \{0, \rho^*(m) \right \}
&|m|= \left (1-\frac{q}{2} \right)\lambda^{\frac{1}{2-q}}\left(1-q\right)^{\frac{q-1}{2-q}},
\end{cases}
\end{equation}
where $\rho^*$(m) is the possible unique minimizer other than $0$.
\label{theorem:uniqueminscalar}
\end{theorem}
\noindent

\begin{figure*}
    \centering

    \begin{subfigure}[t]{0.30\textwidth}
        \centering
        \includegraphics[width=\linewidth]{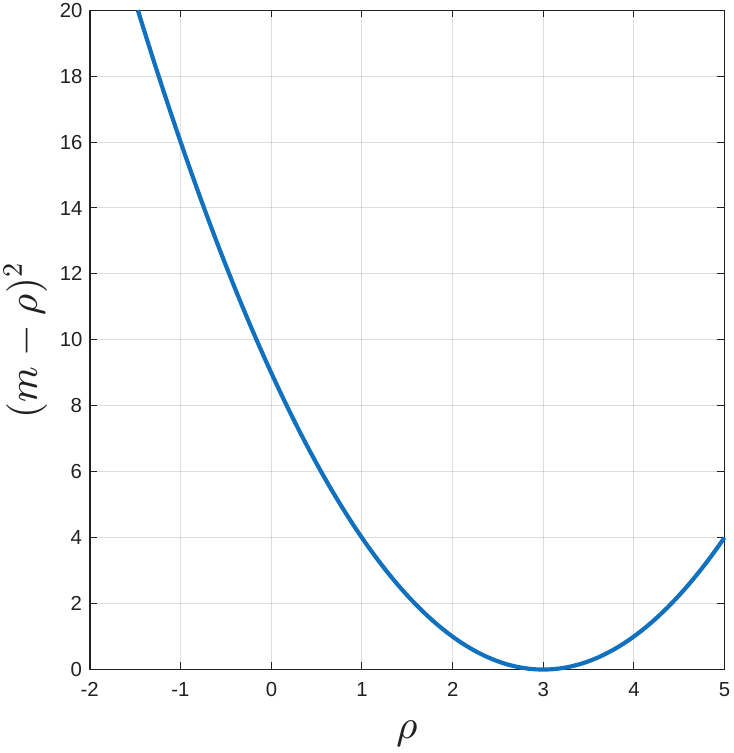}
        \caption{$D(\mathbf{\rho}) = (3-\rho)^2$.}
        \label{fig:fig1a}
    \end{subfigure}
    \hfill
    \begin{subfigure}[t]{0.30\textwidth}
        \centering
        \includegraphics[width=\linewidth]{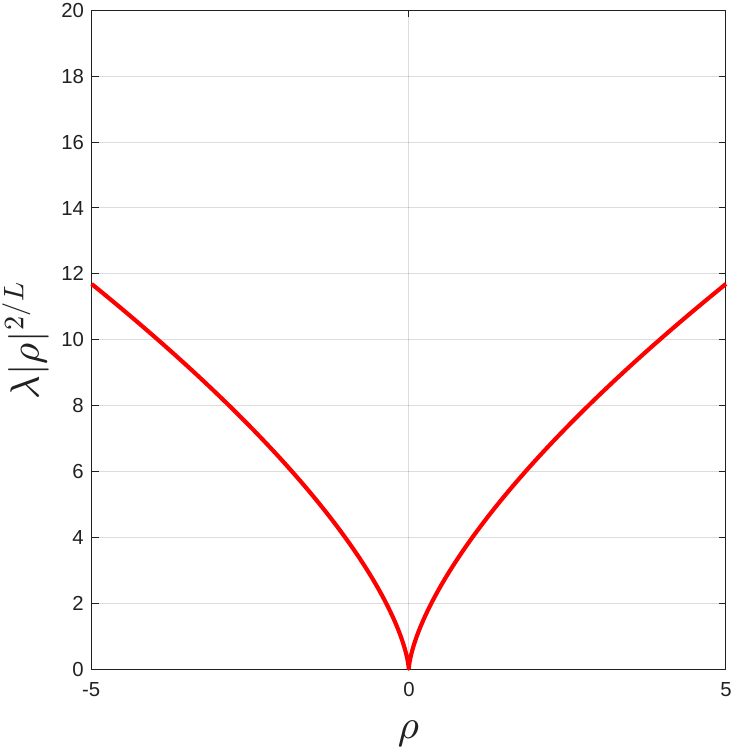}
        \caption{$R(\mathbf{\rho})=4|\rho|^{2/3}$.}
        \label{fig:fig1b}
    \end{subfigure}
    \hfill
    \begin{subfigure}[t]{0.30\textwidth}
        \centering
        \includegraphics[width=\linewidth]{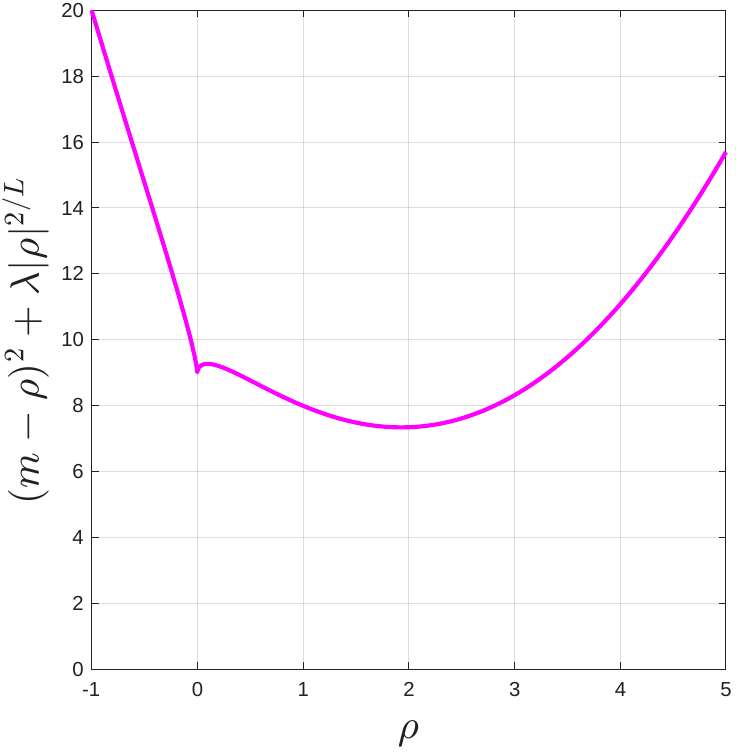}
        \caption{Objective function $\phi(\rho)$.}
        \label{fig:fig1c}
    \end{subfigure}
    \caption{Behavior of the data-fitting term $D(\rho)$, the regularization term $R(\rho)$, and the optimization objective $\phi(\rho)$ for a depth-$3$ factorization of $3$ with $\lambda=4$.}
    \label{fig:figure1}
\end{figure*}

\begin{proof}
As we formally stated in Theorem~\ref{chenetal}, a recent observation by \citet[Lemma~3.2]{chen2025complete} showed that, at each minimizer of $\mathcal{L}$, layers (factors) are balanced across all minimizers, i.e., $|w^*_L| = |w^*_{L-1}| = \cdots = |w^*_1|$ for all $\mathbf{w}^* \in \Omega$. Because of this, we can restrict the optimization to a subspace of $\mathbb{R}^L$ where all points are balanced and formulate an equivalent optimization problem to~\eqref{oopregsc}; that is
\begin{equation}
 \min_{\rho \in \mathbb{R}} \underbrace{(m- \rho)^2}_{D(\rho)} + \underbrace{\lambda |\rho| ^{2/L}}_{R(\rho)}, 
 \label{newformulation}
\end{equation}
where $\rho$ denotes the end-to-end product. An instance of this objective is shown in Fig. \ref{fig:figure1}. Therefore, to prove Theorem~\ref{theorem:uniqueminscalar}, it is sufficient to show that the minimizer of \eqref{newformulation} is unique. Note that for $L \in \{1,2\}$, the objective function is strictly convex; therefore, the optimal solution $\rho^*$ is unique. Hence, it is trivial to examine the case where $L \in \{1,2\}$. Now, suppose $L \geq 3$. For the sake of the method of exhaustion, we investigate the behavior of the solution set of the new formulation case by case. We first assume that $m = 0$. Then 
\begin{equation}
    \rho^2 +\lambda |\rho|^{2/L} \geq 0.
\end{equation}
Since $\rho^2$ and $|\rho|^{2/L}$ are both nonnegative, the lower bound is achieved if and only if $\rho = 0$. Therefore, $\rho^*$ is unique. If $m \neq 0$, we can investigate the problem in two cases, where $m > 0$ and $m < 0$. Before delving into our analysis, to simplify the notation, let us define $\phi:\mathbb{R} \rightarrow \mathbb{R}$ and $q$ such that $\phi(\rho) := (m- \rho)^2 + \lambda |\rho| ^{q}$ and $q := 2/L$, where $L \geq 3$. 

\begin{figure*}
    \centering

    \begin{subfigure}[t]{0.49\textwidth}
        \centering
        \includegraphics[width=\linewidth]{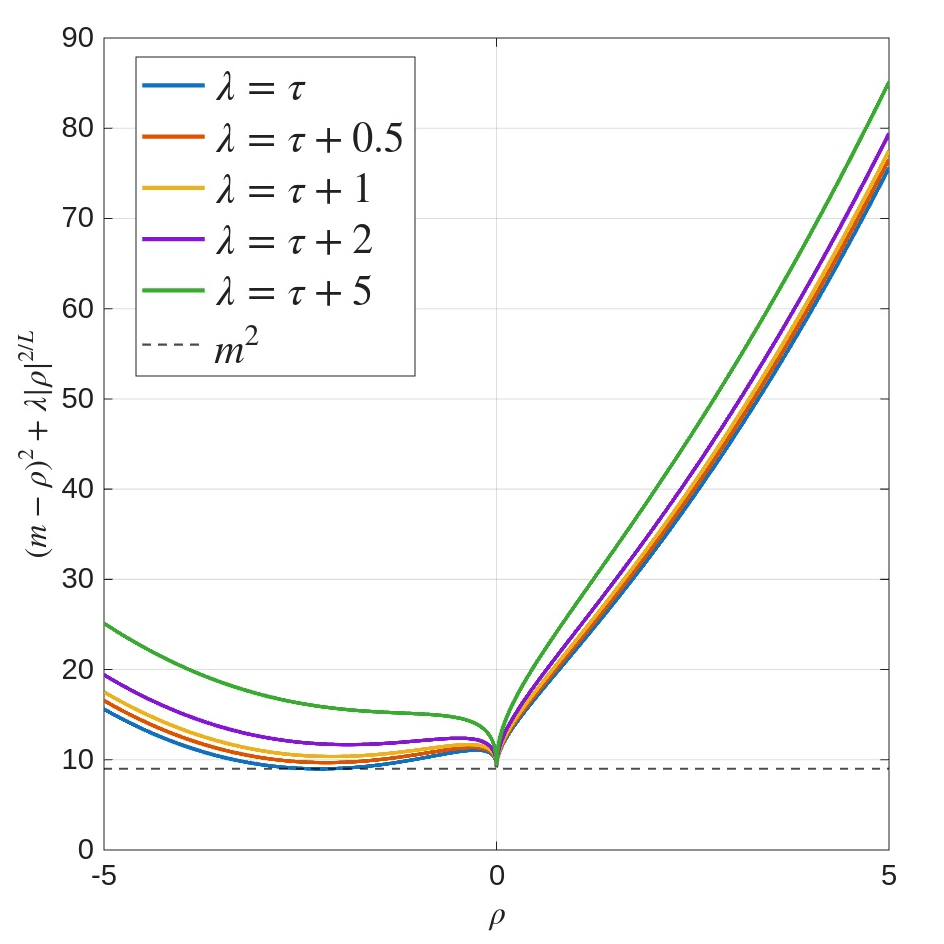}
    \end{subfigure}
    \hfill
    \begin{subfigure}[t]{0.49\textwidth}
        \centering
        \includegraphics[width=\linewidth]{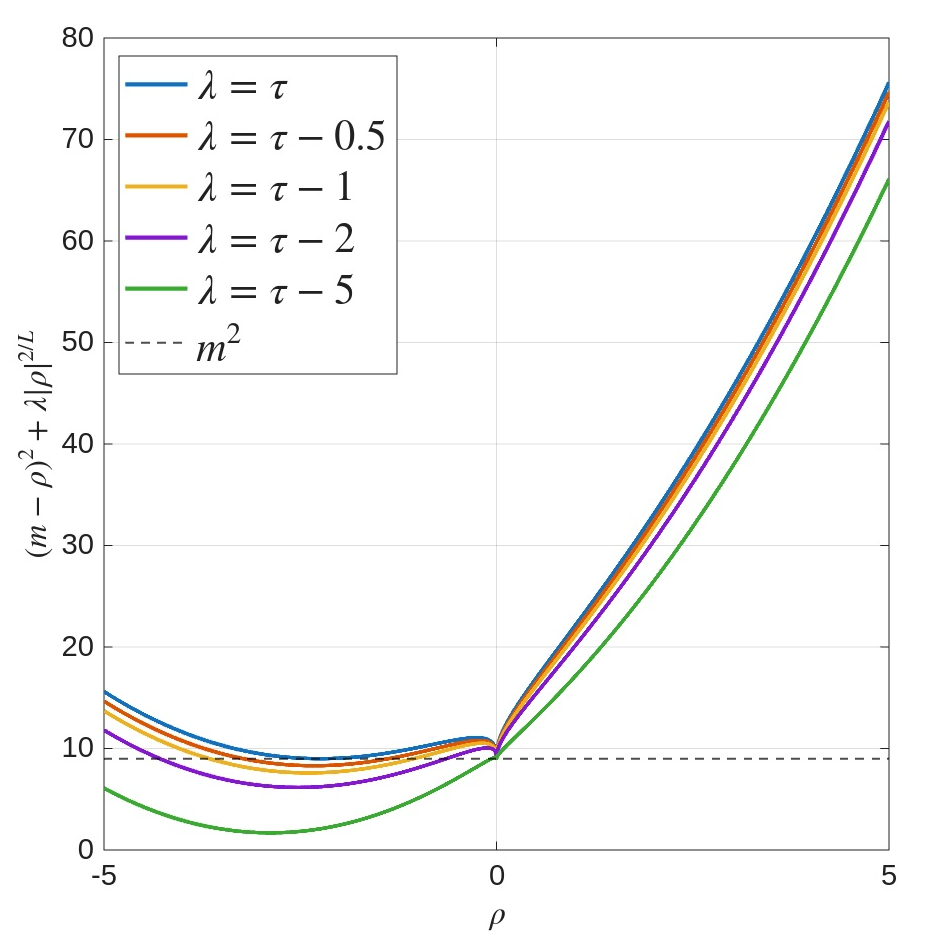}
    \end{subfigure}
    \caption{Behavior of $\phi(\rho)$ for a depth-$5$ factorization of $-3$ under different regularization parameters. Threshold is computed as $\tau = \left(|m|/\left (1-\frac{q}{2} \right)\left(1-q\right)^{({q-1})/({2-q})}\right)^{(2-q)}$, where $q := 2/L$.}
    \label{fig:figure2}
\end{figure*}

Now, we investigate the case where $m > 0$. For any minimizer $\rho^*$, we have that
\begin{equation}
    \phi(\rho^*) - \phi(-\rho^*) \leq 0.
\end{equation}
Suppose that $\rho^* < 0$. Then,
\begin{equation}
   \phi(\rho^*) - \phi(-\rho^*) = (m-\rho^*)^2 + \lambda|\rho^*|^q - (m+\rho^*)^2 - \lambda|\rho^*|^q = -4m\rho^* > 0, 
\end{equation}
which is a contradiction. Therefore, for the case $m > 0$, $\rho^*$ must be larger than or equal to $0$. We now consider the case $m < 0$. Suppose that $\rho^* > 0$. By symmetry, this also leads to a contradiction. Hence, for any minimizer $\rho^*$, we have that
\begin{equation}
    \operatorname{sign}(\rho^*) = \left \{ \operatorname{sign}(m), 0 \right \}.
    \label{signn}
\end{equation}
Therefore, for any $m \geq 0$, it is sufficient to prove that $\phi(\rho)$ has a unique minimizer over $\rho \geq 0$. Furthermore, note that $|\rho|^q$ is symmetric as shown in Fig. \ref{fig:fig1b}, and for any $m \in \mathbb{R}$ such that $f(\rho) = (\rho-m)^2$, we have $f(-\rho) = (\rho+m)^2$. This means that showing $\phi(\rho)$ with $m \geq 0$  has a unique minimizer over $\rho \geq 0$ is sufficient to prove that $\phi(\rho)$ has a unique minimizer for all $m \in \mathbb{R}$. Hence, let us define
\begin{equation}
    \phi(\rho) = (m-\rho)^2 + \lambda \rho^q, \quad \rho, m > 0,
\end{equation}
where $0 < q  < 1$. Note that
\begin{align}
    &\lim_{\rho \rightarrow 0} \phi (\rho) = m^2 \quad \text{and} \quad \lim_{\rho \rightarrow \infty} \phi(\rho) = \infty,\\ 
    &\lim_{\rho \rightarrow 0} \phi' (\rho) = \infty \quad \text{and} \quad \lim_{\rho \rightarrow \infty} \phi'(\rho) = \infty, \\ 
    &\lim_{\rho \rightarrow 0} \phi'' (\rho) = -\infty \quad \text{and} \quad \lim_{\rho \rightarrow \infty} \phi''(\rho) = 2, \label{according} \\
    & \phi'''(\rho) = \lambda q(q-1)(q-2)\rho^{q-3}.
\end{align}
Note that $\phi'(\rho)$ is strictly convex and $\phi''(\rho)$ is strictly increasing since $\phi'''(\rho)$ is positive for all $\rho > 0$. Therefore, $\phi''(\rho)$ intersects the $\rho$-axis at a single point. This implies that $\phi'(\rho)$ has a single local minimum and intersects the $\rho$-axis at most at two points. If $\phi'(\rho)$ intersects the $\rho$-axis at most at one point, i.e.,  $\phi'(\rho) \geq 0$ for all $\rho > 0$, then $\phi(\rho)$ is monotonically increasing and $\rho^* = 0$.  This implies that $\rho^* = 0$ is the unique minimizer. Otherwise, $\phi(\rho)$ has one local maximum and one local minimum for $\rho > 0$, respectively. Let us denote the local minimum by $\rho_m$. Then

\begin{equation}
\arg\min_{\rho \geq 0} \phi(\rho) = 
\begin{cases}
\rho_m 
&\phi(\rho_m) < \phi(0),\\
0
&\phi(\rho_m) > \phi(0),\\

\left \{0, \rho_m \right \}
&\phi(\rho_m) = \phi(0).
\end{cases}
\end{equation}
Now, we investigate the case where $\phi(\rho_m) = \phi(0)$. We know that any critical point $\rho_*$ of $\phi(\rho)$ satisfies $\phi '(\rho_*) = 0$. Therefore, we have two equations to solve as follows. 
\begin{equation}
    -2(m-\rho_m) + \lambda q \rho_m^{q-1} = 0 \quad \text{and} \quad m^2 = (m-\rho_m)^2 + \lambda \rho_m^{q}.
\end{equation}
These equations lead to 

\begin{equation}
    \rho_m = \left ( \lambda(1-q)\right )^{1/(2-q)} \quad \text{and} \quad m= \frac{2-q}{2(1-q)}\rho_m. 
\end{equation}
Hence, $\phi(\rho)$ has a unique minimizer if and only if $m \neq \frac{2-q}{2(1-q)}\rho_m$. Note that when $m \leq 0$, the minimizer is unique if and only if $m \neq - \frac{2-q}{2(1-q)}\rho_m$. Therefore, for any $m \in \mathbb{R} \setminus \left \{\pm\frac{2-q}{2(1-q)}\rho_m  \right \}$, $\phi(\rho)$ has a unique minimizer. Furthermore, this implies that 

\begin{equation}
\arg\min_{\rho \in \mathbb{R}} \phi(\rho) = 
\begin{cases}
0
&|m| < \left (1-\frac{q}{2} \right)\lambda^{\frac{1}{2-q}}\left(1-q\right)^{\frac{q-1}{2-q}},\\
\rho^*(m)
&|m| > \left (1-\frac{q}{2} \right)\lambda^{\frac{1}{2-q}}\left(1-q\right)^{\frac{q-1}{2-q}},\\

\left \{0, \rho^*(m) \right \}
&|m|= \left (1-\frac{q}{2} \right)\lambda^{\frac{1}{2-q}}\left(1-q\right)^{\frac{q-1}{2-q}}.
\end{cases}
\end{equation}
where $\rho^*(m)$ is the possible unique minimizer other than $0$.
\end{proof}
As shown in Fig. \ref{fig:figure2}, when $|m| = \left \{\pm\left (1-\frac{q}{2} \right)\lambda^{\frac{1}{2-q}}\left(1-q\right)^{\frac{q-1}{2-q}}  \right \}$, the end-to-end minimizer is not unique. This result was also observed by \citet{chen2016computing} in the context of computing the proximal operator for the $\ell^p$ quasi-norm. Furthermore, this result leads to a remarkable observation about the geometry of the loss landscape of the $\ell^2$-regularized deep scalar factorization problem near a minimizer. 

\subsubsection{Hessian Spectrum is Constant Across All Minimizers}

Despite the fact that deep learning architectures are heavily overparameterized, they are quite capable of finding benign solutions. This success is usually attributed to the \emph{implicit} regularization mechanisms within neural network training. The implicit regularization refers to the regularization imposed on the optimization objective by the optimization algorithm or the model architecture. Therefore, tuning the level of the implicit regularization imposed on the optimization objective is much more complicated than tuning the level of \emph{explicit regularization}. The explicit regularization is generally realized by incorporating a complexity penalty into the loss function scaled by a hyperparameter. Therefore, tuning the level of explicit regularization boils down to tuning the regularization parameter.

Most notably, one of the implicit regularization mechanisms within neural network training is associated with the \emph{dynamical stability} of gradient-based optimization algorithms near minima. It has been shown that \emph{sharp minima} can be dynamically unstable for gradient-based methods; therefore, they avoid sharp minima and converge to \emph{flat minima} \citep{Wustab}, and flat minima generalize well \citep{hochreiter1997flat}. However, the link between flatness and generalization has been obfuscated: In a large-scale empirical analysis, \citet{Jiang2020Fantastic} evaluated a range of complexity measures for deep networks and observed that sharpness-based measures were most strongly correlated with generalization. However, it is important to note that correlation does not imply causation. Moreover, there is also theoretical evidence for this phenomenon in low-rank matrix recovery~\citep{ding2024flat}. On the other hand, \citet{sharpminimacangeneralize} demonstrated that even good minima in deep neural networks can be arbitrarily sharp, and \citet{xu2026doessgdseekflatness} recently showed that stochastic gradient descent (SGD) does not seek flat minima intrinsically. 

To understand the relationship between generalization ability and flatness, several metrics have been proposed to measure the sharpness/flatness of a minimum. One class of such metrics is the \emph{Hessian-based} measures. There are two measures in this class: the maximum eigenvalue of the Hessian matrix of the loss (worst-case sharpness) and its trace (average sharpness). Unfortunately, there are no exact expressions for these metrics in general settings. To this end, \citet{kamber2026sharpnessminimadeepmatrix} presented an exact expression for the maximum eigenvalue of the Hessian matrix of the squared-error loss evaluated at any minimizer in deep matrix factorization problems. \cite{mulayoff} derived an exact expression for the maximum eigenvalue of the Hessian at flat minima for deep linear neural networks. \citet{ethpaper} provided a full characterization of the Hessian spectrum at an arbitrary point in parameter space for one-hidden layer linear and ReLU networks designed for scalar regression. 

\Cref{theorem:uniqueminscalar} enables us to characterize the full Hessian spectrum of the problem of deep scalar factorization with explicit $\ell^2$ regularization. Moreover, we show that the Hessian spectrum is constant across all minima, and that both the maximum eigenvalue of the Hessian and its trace depend on the depth, the magnitude of the optimal layers, and the regularization parameter. This implies that under Hessian-based measures, \emph{all global minima are equally flat almost always} in the regularized deep scalar factorization problem. To the best of our knowledge, \Cref{theorem:hessianspectrumisconstant} provides the first complete characterization of the Hessian spectrum across minima in deep-factorization-type problems.
\begin{theorem}
Consider the deep scalar factorization objective 
\begin{equation}
    \mathcal{L}(\mathbf{w}) := (m-w_L w_{L-1}\cdots w_1)^2 + \frac{\lambda}{L}  \sum_{i=1}^{L} w_i^2,
\end{equation}
where $w_1,w_2,\cdots,w_L \in \mathbb{R}$ and $\mathbf{w} = (w_1, \ldots, w_L) \in \mathbb{R}^L$. Define $q := 2/L$. Suppose $m \in \mathbb{R} \setminus \left \{\pm\left (1-\frac{q}{2} \right)\lambda^{\frac{1}{2-q}}\left(1-q\right)^{\frac{q-1}{2-q}}  \right \}$. Then, for any $\mathbf{w}^* \in \Omega$, the spectrum of $\nabla^2 \mathcal{L}(\mathbf{w}^*)$ is constant over the solution set $\Omega$. In particular, the eigenvalues are
\begin{equation}
    \lambda_1 = \lambda_2 = \cdots = \lambda_{L-1} = \frac{4\lambda}{L}, \quad \lambda_L = 2Lw^{2L-2}+\frac{4\lambda}{L}-2\lambda,
\end{equation}
where $w = |w^*_1| = \cdots = |w_L^*|$. In which case, for any $\mathbf{w}^* \in \Omega$, 
\begin{equation}
    \lambda_{\max}(\nabla^2 \mathcal{L}(\mathbf{w}^*)) = \max \Bigl \{2Lw^{2L-2}-2\lambda,0 \Bigr \} + \frac{4\lambda}{L}.
\end{equation}   
\label{theorem:hessianspectrumisconstant}
\end{theorem}

\begin{proof}
Differentiating $\mathcal{L}(\mathbf{w})$ with respect to $w_j$, $j \in [L]$, yields
\begin{equation}
\frac{\partial \mathcal{L}(\mathbf w)}{\partial w_j}
=
-2\left(m - \prod_{i=1}^L w_i\right)\prod_{\substack{i=1\\i\neq j}}^L w_i
+\frac{2\lambda}{L}w_j.
\end{equation}
Next, for any \(\mathbf w ^*\in\Omega\), we have that $\nabla \mathcal{L}(\mathbf{w}^*) = 0$. Hence, for any $j \in [L]$, we have that
\begin{equation}
-2\left(m - \prod_{i=1}^L  w_i ^*\right)\prod_{\substack{i=1\\i\neq j}}^L w_i^*
+
\frac{2\lambda}{L} w_j^*
=
0.
\end{equation}
Let $\rho(\mathbf{w}) := \prod_{i=1}^L w_i$. By Theorem~\ref{theorem:uniqueminscalar}, we have that $\rho(\mathbf{w})$ is constant for any $\mathbf{w} \in \Omega$. Write $\rho^* := \rho(\mathbf{w}^*)$ with $\mathbf{w}^* \in \Omega$. Then,
\begin{equation}
2(m-\rho^*)\frac{\rho^*}{w_j^*}
=
\frac{2\lambda}{L} w_j ^*.
\end{equation}
For any $\mathbf{w} \in \mathbb{R}^L$, we have that
\begin{equation}
\frac{\partial^2 \mathcal{L}({\mathbf w})}{\partial w_j\partial w_k}
=
\begin{cases}
2\displaystyle\left(\frac{\rho(\mathbf{w})}{w_{j}} \right)^2
+\dfrac{2\lambda}{L}, 
& j = k,\\[1.5em]
2\dfrac{\rho(\mathbf{w})^2}{w_jw_k}
-
2(m - \rho(\mathbf{w}))\dfrac{\rho(\mathbf{w})}{w_jw_k},
& j\neq k.
\end{cases}
\end{equation}
Now, we fix $\mathbf{w}^* \in \Omega$ and define
\begin{equation}
\mathbf{s}^* := (\operatorname{sign}(w_1^*),\ldots,\operatorname{sign}(w_L^*)) \in \mathbb{R}^L.
\end{equation}
By Theorem~\ref{chenetal}, we know that $|w^*_1| = \cdots = |w_L^*| =: w$, where we note that $w \in \mathbb{R}_{\geq 0}$ is the same across all minimizers by Theorem~\ref{theorem:uniqueminscalar}. Then,
\begin{equation}
\bigl[\nabla^2 \mathcal{L}(\mathbf w^*)\bigr]_{j,k} = 
\begin{cases}
\dfrac{2{\rho^*}^2}{w^2}
+\dfrac{2\lambda}{L}, 
& j = k,\\[1em]
\dfrac{C}{w^2s_js_k},
& j \neq k,
\end{cases}
\end{equation}
where $C = 4{\rho^*}^2 - 2m{\rho^*}$. Since \(\nabla^2 \mathcal{L}(\mathbf w^*)\) is a real symmetric matrix, it admits an orthonormal basis of eigenvectors \(\{\mathbf{v}_i\}_{i=1}^L\) with corresponding eigenvalues \(\{\lambda_i\}_{i=1}^L\).  In particular,
\begin{equation}
\nabla^2 \mathcal{L}(\mathbf w^*)\,\mathbf{v}_i \;=\;\lambda_i\,\mathbf{v}_i,
\end{equation}
and 
\begin{equation}
\bigl[\nabla^2 \mathcal{L}(\mathbf{w}^*)\,\mathbf{v}_i\bigr]_j
= \lambda_iv_{ij}
= \left(\frac{2{\rho^*}^2}{w^{2}}+\frac{2\lambda}{L}\right)v_{ij}
+\sum_{\substack{k=1\\k\neq j}}^L\frac{C}{w^{2}s_js_k}v_{ik}.
\end{equation}
Since $|s_j| = 1$, we can move $s_j$ to the numerator. Therefore,
\begin{equation}
\lambda_iv_{ij}
= \left(\frac{2{\rho^*}^2}{w^{2}}+\frac{2\lambda}{L}\right)v_{ij}
+\sum_{\substack{k=1\\k\neq j}}^L\frac{Cs_js_k}{w^{2}}v_{ik}.
\end{equation}
Furthermore, we can express the right-hand side as a dot product, i.e.,
\begin{align}
\lambda_iv_{ij}
&= \left(\frac{2{\rho^*}^2}{w^{2}}+\frac{2\lambda}{L}\right)v_{ij}
+\frac{Cs_j}{w^{2}}\left(\mathbf{s}^{*\top}\mathbf{v}_i - s_j\,v_{ij}\right) \\
&= \left(\frac{2{\rho^*}^2}{w^{2}}+\frac{2\lambda}{L}\right)v_{ij}
+\frac{Cs_j}{w^{2}}\mathbf{s}^{*\top}\mathbf{v}_i - \frac{C}{w^{2}} v_{ij}.
\end{align}
We can rearrange the terms on the right-hand side to find that
\begin{equation}
v_{ij}\left(\lambda_i
-\left(\frac{2{\rho^*}^2}{w^{2}}+\frac{2\lambda}{L}\right)
+\frac{C}{w^{2}}\right)
=\frac{C}{w^{2}}s_j\mathbf{s}^{*\top}\mathbf{v}_i.
\end{equation}
Next, define
\begin{equation}
\tilde\lambda_i
:=\lambda_i
-\left(\frac{2{\rho^*}^2}{w^{2}}+\frac{2\lambda}{L}\right)
+\frac{C}{w^{2}}.
\end{equation}
Without loss of generality, assume that $C / w^{2} \neq 0$. Then, we obtain
\begin{equation}
\mathbf{v}_i\,\frac{w^{2}\tilde\lambda_i}{C}
=\,\mathbf{s}^*\,\mathbf{s}^{*\top}\,\mathbf{v}_i.
\end{equation}
This implies that \(\mathbf{v}_i\) is an eigenvector of \(\mathbf{s}^*\,\mathbf{s}^{*\top}\) with corresponding eigenvalue \((w^{2}\tilde\lambda_i) / C\) for all $i \in [L]$. Therefore, we have the eigendecomposition 
\begin{equation}
\mathbf{s}^*\,\mathbf{s}^{*\top}
=
\begin{bmatrix}\displaystyle\frac{\mathbf{s}^*}{\|\mathbf{s}^*\|_2}&\mathbf{v}_2&\cdots&\mathbf{v}_L\end{bmatrix}
\begin{bmatrix}\|\mathbf{s}^*\|_2^2&0&\cdots&0\\
0&0&\cdots&0\\
\vdots&\vdots&\ddots&\vdots\\
0&0&\cdots&0\end{bmatrix}
\begin{bmatrix}\displaystyle\frac{\mathbf{s}^*}{\|\mathbf{s}^*\|_2}&\mathbf{v}_2&\cdots&\mathbf{v}_L\end{bmatrix}^{\top}.
\end{equation}
Observe that the only nonzero eigenvalue of \(\mathbf{s}^*\,\mathbf{s}^{*\top}\) is \(\|\mathbf{s}^*\|_2^2=L\); note that this equality holds for \emph{any minimizer} $\mathbf{w}^* \in \Omega$.  Without loss of generality, we identify this eigenvalue with $i = 1$, i.e., $(w^2 \tilde{\lambda}_1) / C = L$. Therefore, for any minimizer, we must satisfy the two equalities
\begin{equation}
 \frac{w^{2}\tilde\lambda_i}{C} = 0 \quad\text{and}\quad  \frac{w^{2}\tilde\lambda_1}{C}= L \quad\Leftrightarrow\quad \tilde\lambda_i =0 \quad \text{and} \quad 
\tilde\lambda_1 =\frac{C}{w^{2}}L, \quad i \in \{2, \ldots, L \}.
\end{equation}
Now, suppose that $i \in \{2, \ldots, L \}$, i.e., $\tilde\lambda_i =0$. Then,
\begin{equation}
\lambda_i
=\Bigl(\frac{2{\rho^*}^2}{w^{2}}+\frac{2\lambda}{L}\Bigr)
-\frac{C}{w^{2}} =\frac{2{\rho^*}^2}{w^{2}}
-\frac{4{\rho^*}^2}{w^{2}}
+\frac{2m\,\rho^*}{w^{2}}
+\frac{2\lambda}{L}
= \frac{2\rho^*(m - \rho^*)}{w^{2}}
+\frac{2\lambda}{L}.
\end{equation}
Since $2(m - \rho^*)\rho^*
=w^{2} (2\lambda) / L$, we get
\begin{equation}
\lambda_i
=\frac{2\lambda}{L}
+\frac{2\lambda}{L}
=\frac{4\lambda}{L}.
\end{equation}
When $i = 1$, i.e., $\tilde\lambda_1 =L{C} / {w^{2}}$, we have that
\begin{equation}
\lambda_1
-\left(\frac{2{\rho^*}^2}{w^{2}}+\frac{2\lambda}{L}\right)
+\frac{C}{w^{2}}
=
L\frac{C}{w^{2}} \quad\Rightarrow\quad \lambda_1 = 2L(w^{2L-2}) + \frac{4\lambda}{L} -2\lambda.
\end{equation}
To complete the proof, we observe that we can write the $\lambda_{\max}$ as follows
\begin{equation}
    \lambda_{\max} = \max\{2L(w^{2L-2})-2\lambda,0 \} + \frac{4\lambda}{L}.
\end{equation} 
\end{proof}

\begin{remark}
Note that when the regularization parameter $\lambda \rightarrow 0$, we have that $\lambda_{i} = 0$ for all $i \in \{2,3,\cdots,L\}$. Furthermore, observe that $w^{L} = |m| =  \sigma_{\max}(m)$, which implies that 
\begin{equation}
\lambda_{\max} = 2L \, {\sigma_{\max}(m)}^{2\left(1-\tfrac{1}{L}\right)}.   
\end{equation}
Note that when the regularization parameter $\lambda \rightarrow 0$, the minimizers of the regularized problem are the $\ell^2$-norm-minimal solutions of the unregularized deep scalar factorization problem. \citet{mulayoff} showed that $\ell^2$-norm-minimal, flat and balanced solutions coincide in the unregularized deep scalar factorization problem with squared-error loss. Therefore, our theorem recovers the result of \citet[Theorem~1]{mulayoff}.
\end{remark}

Before proceeding to the main theorem, we recall an important fact from linear algebra that will be used in the proof of the main theorem to construct a disjoint cover of $\mathbb{R}^{m \times n}$. 

\begin{lemma}
    Suppose $\mathbf{A}, \mathbf{B} \in \mathbb{R}^{m \times n}$. Denote the singular value decompositions of $\mathbf{A}$ and $\mathbf{B}$ by $\mathbf{U}_\mathbf{A} \mathbf{\Sigma}_\mathbf A \mathbf{V}_\mathbf{A}^\top$ and $\mathbf{U}_\mathbf{B}\mathbf{\Sigma}_\mathbf{B} \mathbf{V}_\mathbf{B}^\top$, respectively. If $\mathbf{\Sigma}_\mathbf{A} \neq \mathbf{\Sigma}_\mathbf{B}$, then $\mathbf{A} \neq \mathbf{B}$.
    \label[lemma]{lemma:contradiction}
\end{lemma}

\begin{proof}
    Suppose $\mathbf{A} = \mathbf{B}$ and $\mathbf{\Sigma}_\mathbf{A} \neq \mathbf{\Sigma}_\mathbf{B}$. This implies $\mathbf{A^\top A} = \mathbf{B^\top B}$ such that 
    \begin{equation}
        \mathbf{V}_\mathbf{A} \mathbf{\Sigma}_\mathbf{A}^{\top} \mathbf{\Sigma}_\mathbf{A} \mathbf{V}_\mathbf{A}^\top = \mathbf{V}_\mathbf{B} \mathbf{\Sigma}_\mathbf{B}^{\top} \mathbf{\Sigma}_\mathbf{B} \mathbf{V}_\mathbf{B}^\top.
    \end{equation}
    Note that $\mathbf{\Sigma}_\mathbf{A}^{\top} \mathbf{\Sigma}_\mathbf{A}$ and $\mathbf{\Sigma}_\mathbf{B}^{\top} \mathbf{\Sigma}_\mathbf{B}$ are $n$ by $n$ diagonal matrices. Since $\mathbf{V}_\mathbf{A},\mathbf{V}_\mathbf{B} \in \mathbb{R}^{n \times n}$ are orthogonal matrices, $\mathbf{V}_\mathbf{A}^\top \mathbf{V}_\mathbf{A} = \mathbf{V}_\mathbf{A} \mathbf{V}_\mathbf{A}^\top = \mathbf{V}_\mathbf{B}^\top \mathbf{V}_\mathbf{B} = \mathbf{V}_\mathbf{B} \mathbf{V}_\mathbf{B}^\top = \mathbf{I}$. Denote, $\mathbf{Q} = \mathbf{V}_\mathbf{A}^\top \mathbf{V}_\mathbf{B}$. Note that $\mathbf{Q} \in \mathbb{R}^{n \times n}$ is an orthogonal matrix; hence, $\mathbf{Q}^{-1} = \mathbf{Q^\top}$. Then 
    \begin{equation}
         \mathbf{\Sigma}_\mathbf{A}^{\top} \mathbf{\Sigma}_\mathbf{A}= \mathbf{Q}\mathbf{\Sigma}_\mathbf{B}^{\top} \mathbf{\Sigma}_\mathbf{B}\mathbf{Q}^\top.
    \end{equation}
    This implies that $\mathbf{\Sigma}_\mathbf{A}^{\top} \mathbf{\Sigma}_\mathbf{A}$ is similar\footnote{Two $n$-by-$n$ matrices $\mathbf{A}$ and $\mathbf{B}$ are similar if there exists an invertible $n$-by-$n$ matrix $\mathbf{T}$ such that $\mathbf{B} = \mathbf{T}^{-1} \mathbf{A} \mathbf{T}$.} to $\mathbf{\Sigma}_\mathbf{B}^{\top} \mathbf{\Sigma}_\mathbf{B}$. This is the case if and only if $\mathbf{\Sigma}_\mathbf{A}^{\top} \mathbf{\Sigma}_\mathbf{A} = \mathbf{\Sigma}_\mathbf{B}^{\top} \mathbf{\Sigma}_\mathbf{B}$ if and only if $\mathbf{\Sigma}_\mathbf{A} = \mathbf{\Sigma}_\mathbf{B}$. Contradiction.
\end{proof}

\subsection{Deep Matrix Factorization}

Our analysis of the structure of the set of minima of $\ell^2$-regularized deep matrix factorization problem relies on the von Neumann’s trace inequality. Therefore, before proceeding, we find it useful to first introduce the inequality.

\begin{theorem}
    Let the non-increasingly ordered singular values of $\mathbf{A},\mathbf{B} \in \mathbb{R}^{n \times n}$ be $\sigma_{1}(\mathbf{A}) \geq \cdots \geq \sigma_n(\mathbf{A})$ and $\sigma_{1}(\mathbf{B}) \geq \cdots \geq \sigma_n(\mathbf{B})$. Then 
    \begin{equation}
        \operatorname{tr}(\mathbf{AB}) \leq \sum_{i=1}^n \sigma_{i}(\mathbf{A})\sigma_i(\mathbf{B}),
    \end{equation}
    with equality if and only if $\mathbf{A}$ and $\mathbf{B}^\top$ have the same singular vectors.
    \label{theorem:vonNeumann}
\end{theorem}
\noindent
Understanding the conditions under which von Neumann's trace inequality becomes an equality is critical to our proof technique. Therefore, we provide a clean and structured proof of this inequality in \Cref{sec:proof_of_von_Neumann}.

Until now, we have reformulated the original optimization problem with \Cref{equivalentoptimization}, and we have shown in \Cref{theorem:uniqueminscalar} that the $\ell^2$-regularized deep scalar factorization problem has a unique end-to-end minimizer for all scalars subject to factorization except the two. By leveraging these results together with the von Neumann trace inequality, we show in the following theorem that the $\ell^2$-regularized deep matrix factorization problem with squared-error loss admits a unique end-to-end minimizer for all target matrices that are subject to factorization, except for a set of Lebesgue measure zero determined by the depth and the regularization parameter. 

\begin{theorem}
    Consider the following optimization objective
    \begin{equation}
        \min_{\mathbf{M}\in \mathbb{R}^{d_L\times d_0}} \mathcal{L}_{\mathbf{M}^{\natural}}(\mathbf{M}) := \norm{\mathbf{M}^{\natural}-\mathbf M}_F^2 + \lambda\norm{\mathbf M}_{\mathcal{S}^{2/L}}^{2/L},
    \end{equation}
  where $\mathbf{M}^\natural\in \mathbb{R}^{d_L \times d_0}$ is the target matrix, $\lambda > 0$ is the regularization parameter, and $L \geq 1$ is the depth. Define $\Omega_{\mathbf{M}^\natural} := \arg\min_{\mathbf{M}}\mathcal{L}_{\mathbf{M}^\natural}(\mathbf{M})$, $q := 2/L$, and $r := \min \{m,n \}$. If 
\begin{equation}
    \mathbf{M}^\natural \notin \bigcup_{i=1}^{r}\left \{\mathbf{A} \in\mathbb{R}^{m \times n}:  \sigma({\mathbf{A}})^\top\mathbf{e}_i = \left (1-\frac{q}{2} \right)\lambda^{\frac{1}{2-q}}\left(1-q\right)^{\frac{q-1}{2-q}} \right\},
    \label{theset}
\end{equation} 
then $|\Omega_{\mathbf{M}^\natural}| = 1$.
\label{theorem:maintheorem}
\end{theorem}
\begin{proof}
    For $L \in \{1,2\}$, $\mathcal{L}_{\mathbf{M}^{\natural}}(\mathbf{M})$ is a strictly convex function for all ${\mathbf{M}^{\natural}} \in \mathbb{R}^{d_L \times d_0}$. Therefore, it is trivial to examine since it has a unique minimizer for all ${\mathbf{M}^{\natural}} \in \mathbb{R}^{d_L \times d_0}$. Suppose $L \geq 3$, and define a set $\mathcal{C}^{\mathbf{x}}$ for any $\mathbf{x} \in {\mathbb{R}_+^r}^\downarrow$ such that 
\begin{equation}
    \mathcal{C}^{\mathbf{x}} := \left \{\mathbf{Y} \in \mathbb{R}^{m \times n}: \sigma_1(\mathbf{Y}) = x_1, \sigma_2(\mathbf{Y}) = x_2, \cdots, \sigma_r(\mathbf{Y}) = x_r \right \}.
\end{equation}
By \Cref{lemma:contradiction}, for any $\mathbf{x,y} \in {\mathbb{R}_+^r}^\downarrow$ such that $\mathbf{x} \neq \mathbf{y}$,
\begin{equation}
    \mathcal{C}^{\mathbf{x}} \cap \mathcal{C}^{\mathbf{y}} = \emptyset.
\end{equation}
Furthermore,
\begin{equation}
    \bigcup_{\mathbf{x} \in {\mathbb{R}_+^r}^\downarrow} \mathcal{C}^{\mathbf{x}} = \mathbb{R}^{m \times n}.
\end{equation}
This means that $\left \{\mathcal{C}^{\mathbf{x}} \right \}_{\mathbf{x} \in {\mathbb{R}_+^r}^\downarrow}$ is a disjoint cover of $\mathbb{R}^{m \times n}$. Therefore, for any minimizer $\mathbf{M}^* \in \Omega_{\mathbf{M}^\natural}$, there exists a unique $\mathbf{x}(\mathbf{M}^*) \in {\mathbb{R}_+^r}^{\downarrow}$ such that $\mathbf{M}^{*} \in \mathcal{C}^{\mathbf{x}(\mathbf{M}^*)}$. Therefore, we can rewrite the optimization objective as follows. 
\begin{equation}
    \min_{\mathbf{x}\in {\mathbb{R}_+^r}^\downarrow} \left \{ \min_{\mathbf{M} \in \mathcal{C}^{\mathbf{x}}} \mathcal{L}_{\mathbf{M}^{\natural}}(\mathbf{M}) \right\}.
\end{equation}
Furthermore, a simple matrix algebra reveals that
    \begin{equation}
        \mathcal{L}_{\mathbf{M}^{\natural}}(\mathbf{M}) = \operatorname{tr}({\mathbf{M}^\natural}^\top \mathbf{M}^\natural) - 2\operatorname{tr}(\mathbf{M}^{\natural} \mathbf{M}^\top) + \operatorname{tr}(\mathbf{M^\top M}) + \lambda\norm{\mathbf M}_{\mathcal{S}^{q}}^{q}.
    \end{equation}
Then by von Neumann's trace inequality, for any $\mathbf{x} \in  {\mathbb{R}_+^r}^\downarrow$,
\begin{equation}
     \min_{\mathbf{M} \in \mathcal{C}^{\mathbf{x}}} \mathcal{L}_{\mathbf{M}^{\natural}}(\mathbf{M}) 
    \geq \min_{\mathbf{M} \in \mathcal{C}^{\mathbf{x}}} \left \{ \operatorname{tr}({\mathbf{M}^\natural}^\top \mathbf{M}^\natural) - 2\sum_{i=1}^r \sigma_i({{\mathbf{M}^\natural})}\sigma_i({{\mathbf{M}})} + \operatorname{tr}(\mathbf{M^\top M}) + \lambda\norm{\mathbf M}_{\mathcal{S}^{q}}^{q} \right \} = C.
\end{equation}
Note that the right-hand side of the inequality is constant within $\mathcal{C}^\mathbf{x}$, and this lower bound is achieved if and only if $\mathbf{M}^\natural$ and any minimizer of $\mathcal{L}_{\mathbf{M}^{\natural}}(\mathbf{M})$ within $\mathcal{C}^{\mathbf{x}}$ have the same singular vectors. Now, suppose $\mathbf{M}^*$ is a minimizer of  $\mathcal{L}_{\mathbf{M}^{\natural}}(\mathbf{M})$ within $\mathcal{C}^{\mathbf{x}}$, and the singular vectors of $\mathbf{M}^*$ do not align with the target matrix $\mathbf{M}^\natural$. Denote by $\mathbf{U}_{\mathbf{M}^\natural} \mathbf{\Sigma}_{\mathbf{M}^\natural} \mathbf{V}_{\mathbf{M}^\natural}^\top$ the singular value decomposition of $\mathbf{M}^\natural$. We know that 
\begin{equation}
    \Tilde{\mathbf{M}} := \mathbf{U}_{\mathbf{M}^\natural} 
    \begin{bmatrix}
        \operatorname{diag}(\mathbf{x}) & \mathbf{0}_{r \times (n-r)} \\ 
        \mathbf{0}_{(m-r) \times r} & \mathbf{0}_{(m-r) \times (n-r)}
    \end{bmatrix}
    \mathbf{V}_{\mathbf{M}^\natural}^\top \in \mathcal{C}^\mathbf{x}.
\end{equation}
By von Neumann’s trace inequality,
\begin{equation}
    \mathcal{L}_{\mathbf{M}^{\natural}}(\mathbf{M}^*) >  \mathcal{L}_{\mathbf{M}^{\natural}}(\Tilde{\mathbf{M}}),
\end{equation}
which is a contradiction. Hence, $\mathbf{M}^*$ must have the same singular vectors as the target matrix. This phenomenon is also observed by \citet{lu2015generalized} in the context of generalized singular value thresholding. Furthermore, this means that we can decouple the singular vectors from the singular values and rewrite the optimization problem as follows.
\begin{equation}
    \min_{ \mathbf{x}\in {\mathbb{R}_+^r}} \sum_{i=1}^r \left ( (\sigma_{i}({{\mathbf{M}^\natural}})-x_i)^2 + \lambda x_i^{q} \right).
\end{equation} 
Note that for all $x \in \mathbb{R}_+$ and $i \in [r]$, $(\sigma_{i}({{\mathbf{M}^\natural}})-x_i)^2 + \lambda x_i^{q} \geq 0$. Therefore, the minimizer must be composed of the individual minimizers of $(\sigma_{i}({{\mathbf{M}^\natural}})-x_i)^2 + \lambda x_i^{q}$. Moreover, the individual minimizers are arranged naturally in non-increasing order, as a consequence of the monotonicity of the problem \citep{lu2015generalized}, i.e., 
\begin{equation}
    a \geq b \implies \argmin_{x\in {\mathbb{R}_+}} \{( a-x)^2 + \lambda x^{q}\} \geq \argmin_{x\in {\mathbb{R}_+}} \{(b-x)^2 + \lambda x^{q}\}.
\end{equation} 
Then by \Cref{theorem:uniqueminscalar}, if 
\begin{equation}
    \mathbf{M}^\natural \notin \bigcup_{i=1}^{r}\left \{\mathbf{A} \in\mathbb{R}^{m \times n}:  \sigma({\mathbf{A}})^\top\mathbf{e}_i = \left (1-\frac{q}{2} \right)\lambda^{\frac{1}{2-q}}\left(1-q\right)^{\frac{q-1}{2-q}} \right\},
\end{equation} 
then $\mathbf{x}(\mathbf{M}^*)$ is constant across any $\mathbf{M}^* \in \Omega_{\mathbf{M}^\natural}$. Furthermore, we know that $\mathbf{M}^*$ must align with $\mathbf{M}^\natural$. Therefore,

\begin{equation}
    {\mathbf{M}^*} := \mathbf{U}_{\mathbf{M}^\natural} 
    \begin{bmatrix}
        \operatorname{diag}(\mathbf{x}(\mathbf{M}^*)) & \mathbf{0}_{r \times (n-r)} \\ 
        \mathbf{0}_{(m-r) \times r} & \mathbf{0}_{(m-r) \times (n-r)}
    \end{bmatrix}
    \mathbf{V}_{\mathbf{M}^\natural}^\top.
\end{equation}
So, the natural question to ask here is \emph{``How many different $\mathbf{U}_{\mathbf{M}^\natural}$ and $\mathbf{V}_{\mathbf{M}^\natural}^\top$ are there to generate different $\mathbf{M}^*$?}'' To answer this question, first assume that the singular values of $\mathbf{M}^\natural$ are distinct. This only leads to the incurable sign ambiguity that has no effect on $\mathbf{M}^*$. Therefore, $\mathbf{M}^*$ is unique. Now, suppose that $\mathbf{M}^\natural$ has repeated singular values. Then $\mathbf{M}^*$ must have repeated singular values at precisely the same indices as $\mathbf{M}^\natural$ by monotonicity. Therefore, any rotation ambiguity caused by the repeated singular values does not have any effect on $\mathbf{M}^*$. Therefore, $\mathbf{M}^*$ is unique.
\end{proof}
We now show that the set defined in (\ref{theset}) has Lebesgue measure zero in $\mathbb{R}^{d_L \times d_0}$.
\begin{lemma}
    For any $C \in \mathbb{R}_+$ and index $i \in [\min \{m,n \}]$, the set $S_i(C)$ 

    \begin{equation}
        S_i(C) := \left \{\mathbf{A} \in \mathbb{R}^{m \times n}: \sigma(\mathbf{A})^\top \mathbf{e}_i = C  \right \}
    \end{equation}
     has Lebesgue measure zero in $\mathbb{R}^{m \times n}$.
     \label[lemma]{lebesguemeasurezero}
\end{lemma}

\begin{proof}
    Without loss of generality, assume $m \geq n$, fix a positive scalar $C$ and an index $i \in  [\min \{m,n \}]$. Then for any $\mathbf{A} \in S_i(C)$, define 
    \begin{equation}
        p(\mathbf{A}):= \operatorname{det}(\mathbf{A}^\top \mathbf{A}-C^2 \mathbf{I}).
    \end{equation}
    Note that each entry in $\mathbf{A}^\top \mathbf{A}$ can be expressed as a quadratic polynomial of the entries in $\mathbf{A}$; that is, 
    \begin{equation}
        {(\mathbf{A}^{\top} \mathbf{A})}_{ab} = \sum_{k=1}^n \mathbf{A}_{ak} \mathbf{A}_{kb}. 
    \end{equation}
    Since subtracting $C^2\mathbf{I}$ from $\mathbf{A}^\top \mathbf{A}$ only shifts the diagonal entries, each entry in $\mathbf{A}^\top \mathbf{A}-C^2 \mathbf{I}$ is still a quadratic polynomial of the entries in $\mathbf{A}$. Moreover, by definition, for any $\mathbf{X} = [x_{ij}] \in \mathbb{R}^{n\times n}$, 
    \begin{equation}
        \operatorname{det}(\mathbf{X}) = \sum_{\sigma \in S_n} \operatorname{sgn}(\sigma)\prod_{i=1}^n x_{i\sigma(i)},
    \end{equation}
    where $\operatorname{sgn}$ is the sign function over all possible permutations from $[n]$ to $[n]$, which returns $+1$ if the permutation is even\footnote{A permutation is even if the number of inversions it contains is even, and odd otherwise.}, and $-1$ otherwise. This implies that $p(\mathbf{A})$ is a $2n$-degree polynomial of the entries in $\mathbf{A}$. Note that 
    \begin{equation}
        p(\mathbf{A}) = 0 \iff C^2 \: \: \text{is an eigenvalue of} \: \: \mathbf{A^\top A} \iff C \: \: \text{is a singular value of} \: \: \mathbf{A}.
    \end{equation}
    This means that
    \begin{equation}
        S_i(C) \subseteq \{\mathbf{A} \in \mathbb{R}^{m \times n}: p(\mathbf{A}) = 0\} = \bigcup_{i=1} ^{\min\{m,n\}} S_i(C).
    \end{equation}
    Therefore, $S_i(C) \subseteq \{\mathbf{A} \in \mathbb{R}^{m \times n}: p(\mathbf{A}) = 0\}$. If a real analytic function vanishes on a set of positive measure, then it is identically zero \citep[Chapter~4]{krantz2002primer}. Note that $p$ is the zero polynomial if and only if $\mathbf{A^\top A} = C^2\mathbf{I}$. Therefore, we define $\mathcal{O} \subseteq \{\mathbf{A} \in \mathbb{R}^{m \times n}: p(\mathbf{A}) = 0\}$ such that
    \begin{equation}
        \mathcal{O}:= \left \{C\mathbf{A} \in \mathbb{R}^{m \times n}: \mathbf{A^\top A} = \mathbf{I} \right\}.
    \end{equation}
    Note that $p$ is a nontrivial polynomial on $\mathcal{O}^C$ whereas it is the zero polynomial on $\mathcal{O}$. Hence, $\mathcal{O}^C$ must have Lebesgue measure zero in $\mathbb{R}^{m \times n}$. Furthermore, $\mathcal{O}$ is a scaled \emph{Stiefel} manifold, and 
    \begin{equation}
        \operatorname{dim}(\mathcal{O}) = mn-\frac{n(n+1)}{2}.
    \end{equation}
    This observation implies that $\mathcal{O}$ has Lebesgue measure zero in $\mathbb{R}^{m \times n}$. The union of two sets that have Lebesgue measure zero in $\mathbb{R}^{m \times n}$ has Lebesgue measure zero in $\mathbb{R}^{m \times n}$ . Therefore, $\{\mathbf{A} \in \mathbb{R}^{m \times n}: p(\mathbf{A}) = 0\}$ has Lebesgue measure zero in $\mathbb{R}^{m \times n}$. Since $S_i(C) \subseteq \{\mathbf{A} \in \mathbb{R}^{m \times n}: p(\mathbf{A}) = 0\}$, $S_i(C)$ has Lebesgue measure zero in $\mathbb{R}^{m \times n}$. The same proof technique applies to the case where $m < n$.
    \end{proof}

\begin{remark}
    Since the union of a finite number of sets that have Lebesgue measure zero in $\mathbb{R}^{m \times n}$ individually has Lebesgue measure zero in $\mathbb{R}^{m \times n}$, by \Cref{lebesguemeasurezero}, the set
    \begin{equation}
        \bigcup_{i=1}^{r}\left \{\mathbf{A} \in\mathbb{R}^{m \times n}:  \sigma({\mathbf{A}})^\top\mathbf{e}_i = \left (1-\frac{q}{2} \right)\lambda^{\frac{1}{2-q}}\left(1-q\right)^{\frac{q-1}{2-q}} \right\}
    \end{equation}
    has Lebesgue measure zero in $\mathbb{R}^{m \times n}.$
\end{remark}

As we formalized in \Cref{chenetal}, \citet{chen2025complete} showed that, at any minimizer of the regularized deep matrix factorization problem, the layers are Frobenius-norm balanced. We extend this result as follows.
\begin{corollary}
Consider the optimization problem in (\ref{mainobjectivefunction}). Define $g: \mathbb{R}^N \times [L] \rightarrow \mathbb{R}_+$ such that $g(\mathbf{w},i) = \norm{\mathbf{W}_i}_F$, where $\mathbf{W}_i$ is the $i$th layer (factor). If 
\begin{equation}
        \mathbf{M}^\natural \notin \bigcup_{i=1}^{r}\left \{\mathbf{A} \in\mathbb{R}^{m \times n}:  \sigma({\mathbf{A}})^\top\mathbf{e}_i = \left (1-\frac{q}{2} \right)\lambda^{\frac{1}{2-q}}\left(1-q\right)^{\frac{q-1}{2-q}} \right\},
    \end{equation}
then $g$ is constant over $\Omega \times [L]$.
\label[corollary]{corollary:extended}
\end{corollary}
\begin{proof}
    By definition, for any $\mathbf{w}^* \in \Omega$, $\mathcal{L}(\mathbf{w}^*) = \epsilon_*$, where $\epsilon_*$ is the optimal error. If $\mathbf{M}^\natural$ is not in the Lebesgue measure-zero set, then $D(\mathbf{w}^*)$ is constant across $\Omega$. This implies that $R(\mathbf{w}^*) = \mathcal{L}(\mathbf{w}^*)-D(\mathbf{w}^*)$ is constant across $\Omega$. Since layers (factors) are balanced for any $\mathbf{w}^* \in \Omega$,  $\norm{\mathbf{W}_i^*}_F$ is constant for any $\mathbf{w}^* \in \Omega$ and $i \in [L]$.
\end{proof}

\begin{figure*}
    \centering

    \begin{subfigure}[t]{0.49\textwidth}
        \centering
        \includegraphics[width=\linewidth]{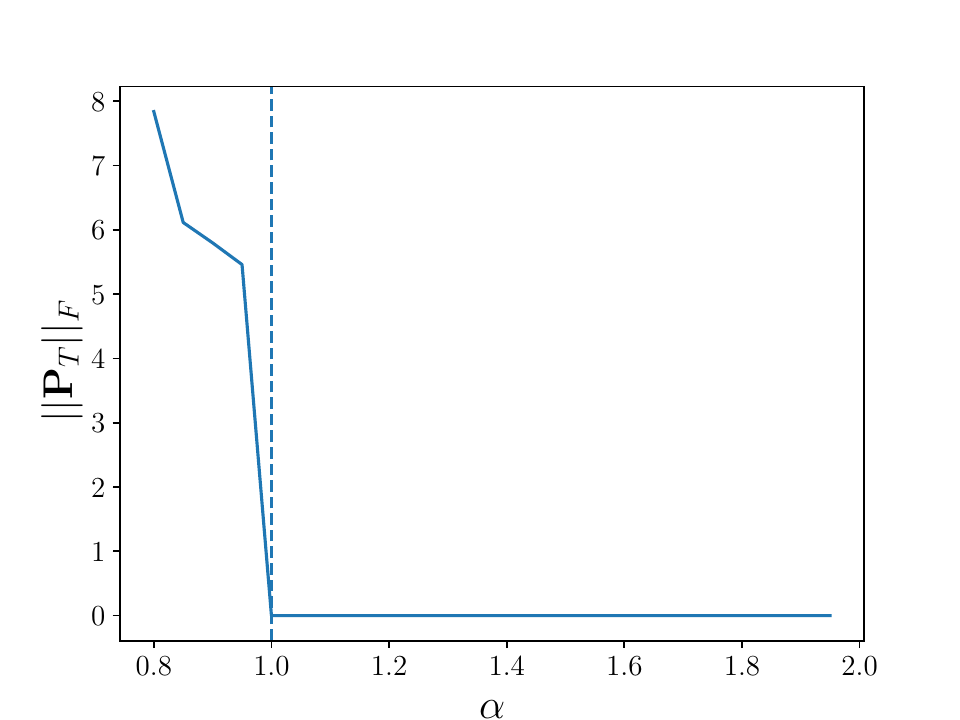}
        \subcaption{$\ell^2$-regularized depth-3 matrix factorization of a Gaussian matrix in $\mathbb{R}^{25 \times 30}$. The hidden layer is a $32$-by-$32$ square matrix.}
    \end{subfigure}
    \hfill
    \begin{subfigure}[t]{0.49\textwidth}
        \centering
        \includegraphics[width=\linewidth]{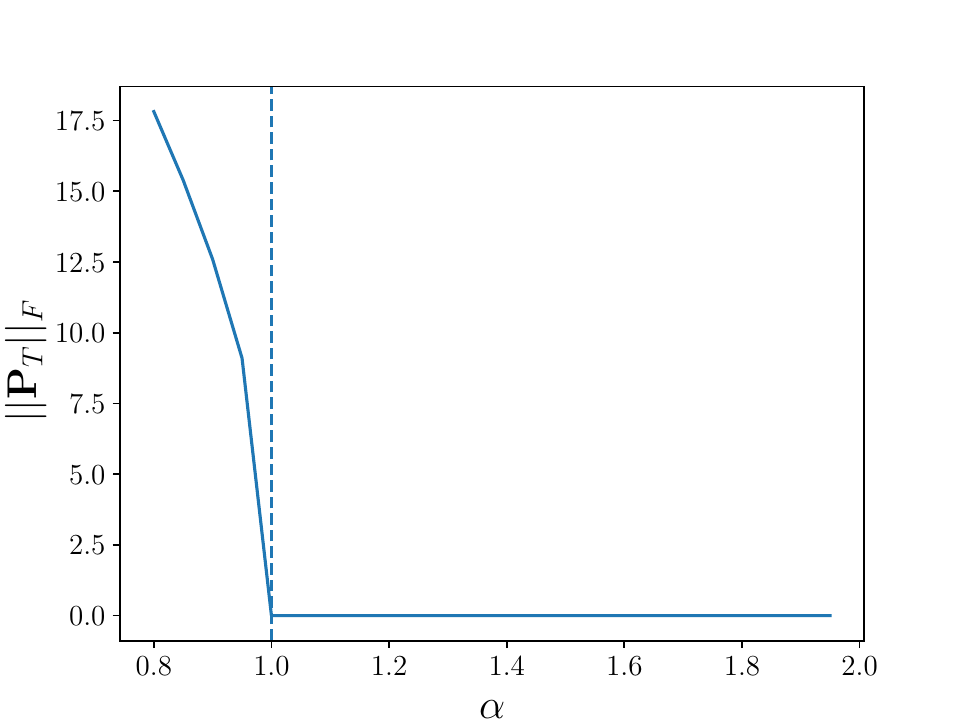}
        \subcaption{$\ell^2$-regularized depth-4 matrix factorization of a Gaussian matrix in $\mathbb{R}^{50 \times 50}$. The hidden layers are $80$-by-$80$ square matrices.}
    \end{subfigure}
    \caption{Frobenius norm of the converged point obtained after $ T=5000$ steps of GD with regularization parameter $\lambda = \alpha \tau$ applied to \eqref{mainobjectivefunction},  where $\tau = \left(\norm{\sigma(\mathbf{M}^\natural)}_{\infty} /\left (1-\frac{q}{2} \right)\left(1-q\right)^{({q-1})/({2-q})}\right)^{(2-q)}$.}
    \label{fig:figure3}
\end{figure*}

\subsection{Collapsing of the Unique Minimum to Zero}

\begin{corollary}
    Consider the optimization problem in~(\ref{mainobjectivefunction}). If $ \lambda > \left(\norm{\sigma(\mathbf{M}^\natural)}_{\infty} / \left (1-\frac{q}{2} \right)\left(1-q\right)^{\frac{q-1}{2-q}}\right)^{2-q}$, then for any $\mathbf{w}^* \in \Omega$, $\prod_{i=1}^L \mathbf{W}_i^* = \mathbf{0}$. 
    \label[corollary]{corollary:collapse}
\end{corollary}
\noindent
As shown in Fig. \ref{fig:figure3}, when the regularization parameter $\lambda$ exceeds the threshold, the end-to-end function after $T$ iterations collapses to $0$. Note that the optimization objective in (\ref{mainobjectivefunction}) has spurious minima when $L \geq 3$ \citep[Theorem~2.2]{chen2025complete}. Furthermore, when

\begin{equation}
\lambda^L \neq \sigma(\mathbf{M}^\natural)_i\left [L^L\left(
\left(\frac{L-2}{L}\right)^{\frac{L}{2(L-1)}}
+
\left(\frac{L}{L-2}\right)^{\frac{L-2}{2(L-1)}}
\right)^{-2(L-1)}\right]  \quad \forall i \in [\min\{m,n\}], 
\end{equation}
we know that every critical point of the regularized deep matrix factorization loss in (\ref{mainobjectivefunction}) is either a local minimizer or a strict saddle point \citep[Corollary~2.3]{chen2025complete}. Under these conditions, the loss landscape is not benign. Therefore, the converged points shown in Fig.~\ref{fig:figure3} are the result of an exhaustive hyperparameter search performed at each regularization level.

\subsection{A Lower Bound for the Trace of the Hessian}
\label{sec:hessian}
By leveraging \Cref{corollary:extended}, we present a lower bound for the trace of the Hessian matrix of the regularized deep matrix factorization loss in (\ref{mainobjectivefunction}) at any minimizer. We leave the question of whether this lower bound is achieved at a flat minimum as an open problem.

\begin{theorem}
    Consider the objective function in (\ref{mainobjectivefunction}). If 
\begin{equation}
        \mathbf{M}^\natural \notin \bigcup_{i=1}^{r}\left \{\mathbf{A} \in\mathbb{R}^{m \times n}:  \sigma({\mathbf{A}})^\top\mathbf{e}_i = \left (1-\frac{q}{2} \right)\lambda^{\frac{1}{2-q}}\left(1-q\right)^{\frac{q-1}{2-q}} \right\},
    \end{equation}
    then for any $\mathbf{w}^* \in \Omega$, 
    \begin{equation}
        \operatorname{tr}(\nabla^2 \mathcal{L}(\mathbf{w}^*)) \geq 2L \frac{\norm{\mathbf{P}^*}^2}{{g^*}^2} + \frac{2\lambda}{L}N,
    \end{equation}
    where $\mathbf{P}^* := \prod_{i=1}^L \mathbf{W}^*_i$ and $g^* = g(\mathbf{w}^*,i)$. Note that $g$ is constant over $\Omega \times [L]$.
    \label{theorem:trace}
\end{theorem}
\begin{proof}
    Note that we can rewrite the optimization objective in (\ref{mainobjectivefunction}) as follows.
    \begin{align}
        \mathcal{L}(\mathbf{w}) &= \operatorname{tr}\left[\left(\mathbf{M}^\natural - \mathbf{A}_k\mathbf{W}_k\mathbf{B}_k\right)^\top
\left(\mathbf{M}^\natural - \mathbf{A}_k\mathbf{W}_k\mathbf{B}_k\right)\right] + \frac{\lambda}{L}\sum_{i=1}^L\operatorname{tr}\left(\mathbf{W}_i^\top\mathbf{W}_i\right) \\  &= 
\operatorname{tr}\bigl({\mathbf{M}^\natural}^\top \mathbf{M}^\natural\bigr)
-2\,\operatorname{tr}\bigl({\mathbf{M}^\natural}^\top\mathbf{A}_k\mathbf{W}_k\mathbf{B}_k\bigr)
+\operatorname{tr}\!\Bigl[(\mathbf{A}_k\mathbf{W}_k\mathbf{B}_k)^\top(\mathbf{A}_k\mathbf{W}_k\mathbf{B}_k)\Bigr]
+\frac{\lambda}{L}\sum_{i=1}^L\operatorname{tr}\bigl(\mathbf{W}_i^\top\mathbf{W}_i\bigr),
    \end{align}
    where $\mathbf{A}_k := \prod_{i=k+1}^L \mathbf{W}_i$ and $\mathbf{B}_k = \prod_{i=1}^{k-1}\mathbf{W}_i$. We know that for any matrices $\mathbf{B},\mathbf{C},\mathbf{X}$ for which the product $\mathbf{B}^\top \mathbf{X}^\top \mathbf{C}\mathbf{X}\mathbf{B}$ exists, the following holds:

\begin{equation}
\frac{\partial}{\partial \mathbf{X}}
\operatorname{tr}\left(\mathbf{B}^\top \mathbf{X}^\top \mathbf{C} \mathbf{X} \mathbf{B}\right) =
\mathbf{C}\,\mathbf{X}\,\mathbf{B}\,\mathbf{B}^\top +
\mathbf{C}^\top\,\mathbf{X}\,\mathbf{B}\,\mathbf{B}^\top.
\end{equation}
Therefore, 
\begin{equation}
    \frac{\partial \mathcal{L}(\mathbf{w})}{\partial \mathbf{W}_k} = 2\mathbf{A}_k^\top \left(\mathbf{A}_k\mathbf{W}_k\mathbf{B}_k - \mathbf{M}^\natural \right)\mathbf{B}_k^\top
\;+\;\frac{2\lambda}{L}\,\mathbf{W}_k.
\end{equation}
By leveraging basic matrix calculus \citep{petersen2008matrix}, the block diagonal entries of the $\nabla^2 \mathcal{L}(\mathbf{w})$ are 
\begin{align}
     \frac{\partial^2\mathcal{L}(\mathbf{w})}{\partial \mathbf{W}_k^2} &= 2\left(\mathbf{B}_k\mathbf{B}_k^\top \otimes \mathbf{A}_k^\top\mathbf{A}_k\right) +\frac{2\lambda}{L}\mathbf{I} \\
     &=  2\left(\mathbf{B}_k\otimes \mathbf{A}_k^\top \right)\left(\mathbf{B}_k^\top\otimes \mathbf{A}_k \right)
+\frac{2\lambda}{L}\mathbf{I},
\end{align}
by using the fact that for any matrices $\mathbf{A}, \mathbf{B}, \mathbf{C}, \mathbf{D}$ such that the matrix products $\mathbf{AB}$ and $\mathbf{CD}$ are well defined, we have
\begin{equation}
(\mathbf{A} \otimes \mathbf{C})(\mathbf{B} \otimes \mathbf{D}) = \mathbf{AB} \otimes \mathbf{CD}.
\end{equation}
Moreover, note that transposition is distributive over the Kronecker product and---unlike the matrix product---order is preserved. This means that for any matrices $\mathbf{X}$ and $\mathbf{Y}$
\begin{equation}
    (\mathbf{X} \otimes \mathbf{Y})^\top = (\mathbf{X}^\top \otimes \mathbf{Y}^\top).
\end{equation}
Thus,
\begin{equation}
    \frac{\partial^2\mathcal{L}(\mathbf{w})}{\partial \mathbf{W}_k^2} = 2\left(\mathbf{B}_k^\top\otimes \mathbf{A}_k \right)^\top \left(\mathbf{B}_k^\top\otimes \mathbf{A}_k \right) +\frac{2\lambda}{L}\mathbf{I}.
\end{equation}
This means that 
\begin{align}
    \operatorname{tr}(\nabla^2 \mathcal{L}(\mathbf{w})) &= \frac{2\lambda}{L}N + \sum_{i=1}^L 2 \norm{\mathbf{B}_i^\top \otimes \mathbf{A}_i}_F^2 \\
    &= \frac{2\lambda}{L}N + \sum_{i=1}^L 2 \norm{\mathbf{B}_i}_F^2 \norm{\mathbf{A}_i}_F^2,
\end{align}
where $N = \sum_{i=1}^{L} d_i d_{i-1}$. Suppose 
\begin{equation}
        \mathbf{M}^\natural \notin \bigcup_{i=1}^{r}\left \{\mathbf{A} \in\mathbb{R}^{m \times n}:  \sigma({\mathbf{A}})^\top\mathbf{e}_i = \left (1-\frac{q}{2} \right)\lambda^{\frac{1}{2-q}}\left(1-q\right)^{\frac{q-1}{2-q}} \right\}.
\end{equation}
Then for any $\mathbf{w}^* \in \Omega$, $\mathbf{A}_i \mathbf{W}_i^* \mathbf{B}_i = \prod_{i=1}^L \mathbf{W}_i^* =: \mathbf{P}^*$, which is constant. Furthermore, the Cauchy-Schwarz inequality states that 
\begin{equation}
    \norm{\mathbf{P}^*}_F= \norm{\mathbf{A}_i \mathbf{W}^*_i \mathbf{B}_i}_F \leq \norm{\mathbf{A}_i}_F \norm{\mathbf{W}^*_i}_F \norm{\mathbf{B}_i}_F.
\end{equation}
Hence for all $i \in [L]$, 
\begin{equation}
   \norm{\mathbf{A}_i}_F^2 \norm{\mathbf{B}_i}_F^2 \geq \frac{\norm{\mathbf{P}^*}_F^2}{\norm{\mathbf{W}^*_i}_F^2}.
\end{equation}
We showed in Corollary \ref{corollary:extended} that $\norm{\mathbf{W}^*_i}$ is constant over $\Omega \times [L]$. Denote $g^* := \norm{\mathbf{W}^*_i}$. This implies that for all $\mathbf{w}^* \in \Omega$,
\begin{equation}
    \operatorname{tr}(\nabla^2 \mathcal{L}(\mathbf{w}^*)) \geq \frac{2\lambda}{L}N + 2L\frac{\norm{\mathbf{P}^*}_F^2}{g*^2}.
\end{equation}
\end{proof}

\section{Discussion}
\label{sec:discussion}
As we showed in \Cref{equivalentoptimization}, in order to understand the structure of the set of functions represented by the minimizers of the regularized deep matrix factorization problem, we can examine the set of minimizers of the Schatten-$2/L$ regularized problem.  Beyond this observation, the Schatten-$p$ regularized problem is of independent interest in low-rank matrix recovery. 

Optimization problems designed to recover an underlying low-rank matrix from a small number of linear measurements incorporate an $\ell^0$ penalty on the singular values to penalize high-rank solutions. However, this problem is NP-hard, and so the $\ell^1$ norm is used as a convex surrogate for the $\ell^0$ norm \citep{candes2005decoding}. Nevertheless, the $\ell^1$ penalty has been observed to be suboptimal for this task, with non-convex penalties usually outperforming $\ell^1$ regularization \citep{candes2008enhancing}. Therefore, \citet{lu2014generalized} proposed to solve the following non-convex low-rank matrix recovery problem.

\begin{equation}
    \min_{\mathbf{X} \in \mathbb{R}^{m \times n}} f(\mathbf{X}) := h(\mathbf{X}) + \sum_{j=1}^{\min \{m,n\}}g(\sigma_i(\mathbf{X})),
\end{equation}
where $g:\mathbb{R}_+ \rightarrow \mathbb{R}_+$ is a continuous, concave, and non-decreasing function. To solve this family of non-convex problems efficiently, \citet{lu2015generalized} proposed the Generalized Proximal Gradient (GPG) method whose update rule takes the following form:

\begin{equation}
    \mathbf{X}_{k+1} = \textbf{Prox}_{\frac{1}{\eta}g}^{\sigma}\left(\mathbf{X}_k - \frac{1}{\eta} \nabla h(\mathbf{X}_k)\right),
\end{equation}
where 
\begin{equation}
    \textbf{Prox}_g^\sigma(\mathbf{M}^\natural) := \argmin_{\mathbf{M}\in \mathbb{R}^{m \times n}} \frac{1}{2}\norm{\mathbf{M}^\natural-\mathbf{M}}_F^2 + \sum_{j=1}^{\min \{m,n\}} g(\sigma_i(\mathbf{M})),
\end{equation}
and $\textbf{Prox}_g^{\sigma}\left (\cdot \right)$ is called the \emph{generalized singular value thresholding operator}. They rigorously proved that when minimizing the proximal operator objective, the singular vectors can be decoupled from the singular values, as the singular vectors of any minimizer of the problem must align with those of the target matrix. They also provided precise conditions under which this problem has a unique minimizer. Furthermore, \citet{chen2016computing} characterized these conditions in the context of computing the proximal operator for the $\ell^p$ quasi-norm, i.e., $g:= |\cdot|^p$. Our work leverages these results to analyze the solution set of overparameterized deep matrix factorization problems, though our proof techniques differ substantially from theirs. Additionally, our analysis reveals remarkable properties of the geometry of the landscape near minima, which is one of the key aspects that determine the implicit bias of gradient-based algorithms toward certain solutions.

\section{Conclusion}
\label{sec:conclusion}
In this paper, we showed that $\ell^2$-regularized deep scalar factorization problems with squared-error loss admit a unique end-to-end minimizer for all scalars except two. By leveraging this, we characterized the full Hessian spectrum of the regularized deep scalar factorization loss at any minimizer. This revealed a fundamental property of the loss landscape of the regularized scalar factorization problem: all global minima are equally flat almost always. Furthermore, we showed that the regularized deep matrix factorization problem admits a unique end-to-end minimizer for all target matrices subject to factorization, except for a set of Lebesgue measure zero determined by the depth and the regularization parameter. This observation revealed a remarkable result: in regularized deep matrix factorization, the Frobenius norm of each layer is the same across all minimizers almost always. From this, we derived a global lower bound for the trace of the Hessian matrix evaluated at any minimizer of the regularized deep matrix factorization problem that holds almost always. Finally, we provided a critical threshold for the level of regularization above which the unique end-to-end minimizer collapses to zero.
\bibliographystyle{plainnat}  
\bibliography{refs}       

\newpage
\appendix
\section{Variational Form of the Schatten-$2/L$ Quasi-Norm}
\label[appendix]{appendix:schattenvariational}
In this section, we develop the ideas behind the variational form of the Schatten-p norm/quasi-norm presented by \citet{schatten1,schatten2}.

\begin{definition}
    Let $\mathbf{x} \in \mathbb {R}^n$ and $\mathbf{x}^{\downarrow}$ be the vector obtained by ordering the entries of $\mathbf{x}$ in descending order. If 
    \begin{equation}
        \sum_{i=1}^n x_i^\downarrow \leq \sum_{i=1}^n y_i^\downarrow,
    \end{equation}
    then $\mathbf{x}$ is \emph{weakly majorized by $\mathbf{y}$}, and denoted by $\mathbf{x} \succ_w \mathbf{y}$. If $\mathbf{x}, \mathbf{y} \in \mathbb{R}_+^n$ and 
    \begin{equation}
        \prod_{i=1}^n x_i^\downarrow \leq \prod_{i=1}^n y_i^\downarrow,
    \end{equation}
    then $\mathbf{x}$ is weakly $\log$-majorized by $\mathbf{y}$, and denoted by $\mathbf{x} \succ_{w\log} \mathbf{y}$.
\end{definition}

\begin{lemma}[Young's Theorem]
Suppose $a,b$ are non-negative scalars and $p,q \geq 1$ such that $1/q + 1/p = 1$. Then
\begin{equation}
    ab \leq \frac{a^p}{p} + \frac{b^q}{q}.
\end{equation}
\label[lemma]{lemma:youngsinequality}
\end{lemma}

\begin{proof}
    Suppose $f:\mathbb{R} \rightarrow \mathbb{R}$ such that $f(x) =  e^x$. Then
    \begin{align}
        ab &= f(\log(a)+\log(b)) \\
        &= f(\frac{p\log(a)}{p} + \frac{q\log(b)}{q}) \\
        &\leq \frac{1}{p}f(p\log(a)) + \frac{1}{q}f(q\log(b)) \\
        &= \frac{1}{p}e^{p\log(a)} + \frac{1}{q}e^{q\log(b)} \\
        &= \frac{1}{p}\left(e^{\log(a)}\right) ^p+ \frac{1}{q}\left(e^{\log(b)}\right) ^q \\ 
        &= \frac{a^p}{p} + \frac{b^q}{q}.
    \end{align}
\end{proof}

\begin{lemma}[Hölder's Inequality]
    For any $p,q \in \mathbb{R}_+$ satisfying $1/p +1/q = 1/r$, and for any $\mathbf{x},\mathbf{y} \in \mathbb R^n$, we have
    \begin{equation}
    \sum_{i=1}^n \abs{x_i y_i}^r  \leq \norm{\mathbf{x}}_p^r \norm{\mathbf{y}}_q^r.
\end{equation}

\end{lemma}

\begin{proof}
    If $1/p + 1/q = 1/r$, then $1/(p/r) + 1/(q/r) = 1$. We also know that 
    \begin{equation}
       \sum_{i=1}^n \frac{\abs{x_i y_i}^r}{\norm{\mathbf x}_p ^r\norm{\mathbf y}_q ^r} = \sum_{i=1}^n \frac{\abs{x_i}^r \abs{y_i}^r}{\norm{\mathbf x}_p ^r \norm{\mathbf y}_q ^r}.
   \end{equation} 
   Suppose $\abs{x_i}^r/\norm{\mathbf{x}}_p^r =: a_i$ and $\abs{y_i}^r/\norm{\mathbf{x}}_q^r=: b_i$. Then, by using \Cref{lemma:youngsinequality},
    \begin{align}
       \sum_{i = 1}^n a_ib_i \leq \sum_{i=1}^n \frac{ra_i^{p/r}}{p} + \frac{r b_i^{q/r}}{q} &:= \sum_{i=1}^n \frac{r\abs{x_i}^p}{p\norm{\mathbf{x}}_p ^p} + \frac{r\abs{y_i}^q}{q\norm{\mathbf{y}}_q ^q}\\
       &=\frac{r}{p}\sum_{i=1}^n \frac{\abs{x_i}^p}{\norm{\mathbf{x}}_p^p} + \frac{r}{q} \sum_{i=1}^n \frac{\abs{y_i}^q}{\norm{\mathbf{y}}_q ^q} \\
       &= \frac{r}{p} + \frac{r}{q} = \frac{1}{p/r} + \frac{1}{q/r} \\
       &= 1.
   \end{align}
\end{proof}

\begin{theorem}
Let $\mathbf{A} \in \mathbb{R}^{m \times k}$, $\mathbf{B} \in \mathbb{R}^{k \times n}$ and $r = \min\{\operatorname{rank}(\mathbf{A}), \operatorname{rank}(\mathbf{B})\}$. Denote by  $\sigma_{1}(\mathbf{A}) \geq \cdots \geq \sigma_r(\mathbf{A}) > 0$ and $\sigma_{1}(\mathbf{B}) \geq \cdots \geq \sigma_r(\mathbf{B}) > 0$ the ordered singular values of $\mathbf{A}$ and $\mathbf{B}$, respectively. Then, for any $p \in \mathbb{R}_+$,
\begin{equation}
    \sum_{i=1}^{r} \sigma_{i}(\mathbf{AB})^p \leq\sum_{i=1}^{r} \sigma_{i}(\mathbf{A})^p\sigma_{i}(\mathbf{B})^p.
\end{equation}
\label{theorem:p}
\end{theorem}

\begin{proof}
By using the fact that $[\sigma_1(\mathbf{AB)}   \cdots \sigma_r(\mathbf{AB})] \succ_{w\log} [\sigma_1(\mathbf{A})\sigma_1(\mathbf{B}) \cdots \sigma_r(\mathbf{A})\sigma_r(\mathbf{B})]$ (Theorem~9.H.1, \cite{marshall1979inequalities}),
\begin{equation}
    \prod_{i=1}^r \sigma_{i}(\mathbf{AB}) \leq \prod_{i=1}^r \sigma_{i}(\mathbf{A}) \sigma_i(\mathbf{B}). 
\end{equation}
Therefore, for any $p \in [0,\infty)$, 
\begin{equation}
    \prod_{i=1}^r \sigma_{i}(\mathbf{AB})^p \leq \prod_{i=1}^r \sigma_{i}(\mathbf{A}) ^p\sigma_i(\mathbf{B})^p. 
\end{equation}
This implies that $\mathbf{\sigma}(\mathbf{AB})^p \in \mathbb{R}^r$ is weakly $\log$-majorized by $\left(\mathbf{\sigma}(\mathbf{A}) \odot \mathbf{\sigma}(\mathbf{B}) \right)^p \in \mathbb{R}^d$. It is well-known that a weak $\log$-majorization implies the weak majorization when the entries in vectors are strictly positive (Theorem~5.A.2.b, \citet{marshall1979inequalities}). Suppose $\operatorname{rank}(\mathbf{AB}) =: \hat{r} \leq r$. Therefore,
\begin{equation}
    \sum_{i=1}^{r} \sigma_{i}(\mathbf{AB})^{p} = \sum_{i=1}^{\hat{r}} \sigma_{i}(\mathbf{AB})^p\leq \sum_{i=1}^{\hat{r}}  \sigma_{i}(\mathbf{A})^{p} \sigma_{i}(\mathbf{B})^{p} \leq \sum_{i=1}^{r} \sigma_{i}(\mathbf{A})  ^p\sigma_{i}(\mathbf{B})^p. 
\end{equation}
\end{proof}

\begin{corollary}
\label[corollary]{corollary:schattenp}
Let $\mathbf{A} \in \mathbb{R}^{m \times k}$, $\mathbf{B} \in \mathbb{R}^{k \times n}$. For any $p,q \in \mathbb{R}_+$ satisfying $1/p +1/q = 1/\gamma$, we have
    \begin{equation}
     \norm{\mathbf{AB}}_{\mathcal{S}^\gamma}^\gamma \leq \norm{\mathbf{A}}_{\mathcal{S}^p}^\gamma \norm{\mathbf{B}}_{\mathcal{S}^q}^\gamma.
\end{equation}
\end{corollary}
    
\begin{proof}
If $1/p + 1/q = 1/\gamma$, then $1/(p/\gamma) + 1/(q/\gamma) = 1$. Define $r := \min \{\operatorname{rank}(\mathbf{A}),\operatorname{rank}(\mathbf{B}) \}$. By Theorem~\ref{theorem:p}, we also know that
\begin{equation}
    \sum_{i=1}^{r}\frac{\sigma_i(\mathbf{AB})^\gamma}{\norm{\mathbf{A}}_{\mathcal{S}^p}^\gamma\norm{\mathbf{B}}_{\mathcal{S}^q}^\gamma} \leq \sum_{i=1}^{r}\frac{\sigma_i(\mathbf{A})^\gamma \sigma_i(\mathbf{B})^\gamma}{\norm{\mathbf{A}}_{\mathcal{S}^p}^\gamma\norm{\mathbf{B}}_{\mathcal{S}^q}^\gamma}.
\end{equation}
Suppose $\sigma_i(\mathbf{A})^\gamma / \norm{\mathbf{A}}_{\mathcal{S}^p}^\gamma  =: a_i$ and $\sigma_i(\mathbf{B})^\gamma / \norm{\mathbf{B}}_{\mathcal{S}^p}^\gamma  =: b_i$. Then, by using \Cref{lemma:youngsinequality},
 \begin{align}
       \sum_{i = 1}^r a_ib_i \leq \sum_{i=1}^r \frac{\gamma a_i^{p/\gamma}}{p} + \frac{\gamma b_i^{q/\gamma}}{q} &:= \sum_{i=1}^r \frac{\gamma \sigma_i({\mathbf{A}})^p}{p\norm{\mathbf{A}}_{\mathcal{S}^p} ^p} + \frac{\gamma \sigma_i({\mathbf{B}})^q}{q\norm{\mathbf{B}}_{\mathcal{S}^q} ^q} \\
       &\leq\frac{\gamma}{p} \sum_{i=1}^{\operatorname{rank}(\mathbf{A})} \frac{ \sigma_i({\mathbf{A}})^p}{\norm{\mathbf{A}}_{\mathcal{S}^p} ^p} + \frac{\gamma}{q}  \sum_{i=1}^{\operatorname{rank}(\mathbf{B})} \frac{ \sigma_i({\mathbf{B}})^q}{\norm{\mathbf{B}}_{\mathcal{S}^q} ^q} \\
       &= \frac{\gamma}{p} + \frac{\gamma}{q} = \frac{1}{p/\gamma} + \frac{1}{q/\gamma} \\
       &= 1
   \end{align}
   
\end{proof}

\begin{remark}
\label[remark]{remark:schattenmulti}
    For any $p_1,p_2,\cdots, p_L \in \mathbb{R}_+$ satisfying $\sum_{i=1}^L 1/p_i = 1/\gamma$, and for any matrices $\mathbf{X}_1, \mathbf{X}_2, \cdots, \mathbf{X}_L$ such that the product $\mathbf{X}_L \mathbf{X}_{L-1} \cdots \mathbf{X}_1$ is feasible, we have
    \begin{equation}
        \norm{\mathbf{X}_L \mathbf{X}_{L-1} \cdots \mathbf{X}_1}_{\mathcal{S}^\gamma}^\gamma \leq \norm{\mathbf{X}_L}_{\mathcal{S}^{p_L}}^\gamma \norm{\mathbf{X}_{L-1}}_{\mathcal{S}^{p_{L-1}}}^\gamma \cdots \norm{\mathbf{X}_1}_{\mathcal{S}^{p_1}}^\gamma.
    \end{equation}
\end{remark}
\begin{proof}
    Let us define $\mathbf{X}_L \mathbf{X}_{L-1} \cdots \mathbf{X}_2 := \hat{\mathbf{X}}$ and $\sum_{i=2}^L 1/p_i = 1/\hat{\gamma}$. Therefore, by \Cref{corollary:schattenp},
    \begin{equation}
    \norm{\hat{\mathbf{X}}\mathbf{X}_1}_{\mathcal{S}^{\gamma}}^\gamma \leq \norm{\hat{\mathbf{X}}}_{\mathcal{S}^{\hat \gamma}}^\gamma \norm{\mathbf{X}_1}_{\mathcal{S}^{p_1}}^\gamma
    \end{equation}
    This is equivalent to
     \begin{equation}
    \norm{\hat{\mathbf{X}}\mathbf{X}_1}_{\mathcal{S}^{\gamma}}^\gamma \leq \left (\norm{\hat{\mathbf{X}}}_{\mathcal{S}^{\hat \gamma}}^ {\hat{\gamma}} \right) ^{\gamma/\hat{\gamma}} \norm{\mathbf{X}_1}_{\mathcal{S}^{p_1}}^\gamma
    \end{equation}
    Now, define $\mathbf{X}_L \mathbf{X}_{L-1} \cdots \mathbf{X}_3 := \hat{\mathbf{Y}}$ and $\sum_{i=3}^L 1/p_i = 1/\hat{\beta}$. Therefore, 
    \begin{align}
    \norm{\hat{\mathbf{X}}\mathbf{X}_1}_{\mathcal{S}^{\gamma}}^\gamma \leq \left (\norm{\hat{\mathbf{Y}} \mathbf{X}_2}_{\mathcal{S}^{\hat \gamma}}^ {\hat{\gamma}} \right) ^{\gamma/\hat{\gamma}} \norm{\mathbf{X}_1}_{\mathcal{S}^{p_1}}^\gamma &\leq \left ( \norm{\hat{\mathbf{Y}}}_{\mathcal{S}^{\hat \beta}}^{\hat{\gamma}} \norm{\mathbf{X}_2}_{\mathcal{S}^{p_2}}^{\hat{\gamma}}\right)^{\gamma/\hat{\gamma}} \norm{\mathbf{X}_1}_{\mathcal{S}^{p_1}}^\gamma \\
    &= \left ( \norm{\hat{\mathbf{Y}}}_{\mathcal{S}^{\hat \beta}}^{\hat{\beta}} \right)^{\gamma/\hat{\beta}} \norm{\mathbf{X}_2}_{\mathcal{S}^{p_2}}^{\gamma}\norm{\mathbf{X}_1}_{\mathcal{S}^{p_1}}^\gamma.
    \end{align}
    By induction, 
    \begin{equation}
        \norm{\mathbf{X}_L \mathbf{X}_{L-1} \cdots \mathbf{X}_1}_{\mathcal{S}^\gamma}^\gamma \leq \norm{\mathbf{X}_L}_{\mathcal{S}^{p_L}}^\gamma \norm{\mathbf{X}_{L-1}}_{\mathcal{S}^{p_{L-1}}}^\gamma \cdots \norm{\mathbf{X}_1}_{\mathcal{S}^{p_1}}^\gamma.
    \end{equation}
\end{proof}
\begin{theorem}
    Let $L \in \mathbb{N}$ such that $L \geq 1$. For any $\mathbf{X} \in \mathbb{R}^{m \times n}$,
    \begin{equation}
        \norm{\mathbf{X}}_{\mathcal{S}^{2/L}}^{2/L} = \min_{\substack{\mathbf{W}_1, \cdots,\mathbf{W}_L: \\
        \mathbf{W}_L \cdots \mathbf{W}_1 = \mathbf{X} }} \frac{1}{L} \sum_{i=1}^L \norm{\mathbf{W}_i}_F^2.
    \end{equation}
    \label{theorem:variationalform}
\end{theorem}

\begin{proof}
Suppose that $\mathbf{X} = \mathbf{W}_L \cdots \mathbf{W}_1$. Define $p_1,p_2,\cdots,p_L$ such that $\sum_{i=1}^L 1/p_i = L/2$. By using \Cref{remark:schattenmulti}, 

\begin{equation}
    \norm{\mathbf{X}}_{\mathcal{S}^{2/L}}^{2/L} \leq \norm{\mathbf{W}_L}_{\mathcal{S}^{p_L}}^{2/L} \norm{\mathbf{W}_{L-1}}_{\mathcal{S}^{p_{L-1}}}^{2/L} \cdots \norm{\mathbf{W}_1}_{\mathcal{S}^{p_1}}^{2/L}.   
\end{equation}
We can determine a specific configuration of $p_i$'s such that $p_1 = p_2 = \cdots = p_L = 2$. Note that Schatten-$2$ norm is equivalent to Frobenius norm. Therefore, 
\begin{align}
    \norm{\mathbf{X}}_{\mathcal{S}^{2/L}}^{2/L} &\leq \left ( \prod_{i=1}^L \norm{\mathbf{W}_i}^{2}\right)^{1/L} \\
    &\leq \frac{1}{L} \sum_{i=1}^L \norm{\mathbf{W}_i}^{2}. 
\end{align}
The second inequality follows from the AM-GM inequality. This implies that 
\begin{equation}
    \norm{\mathbf{X}}_{\mathcal{S}^{2/L}}^{2/L} \leq  \min_{\substack{\mathbf{W}_1, \cdots,\mathbf{W}_L: \\
         \mathbf{W}_L \cdots \mathbf{W}_1 = \mathbf{X} }} \frac{1}{L} \sum_{i=1}^L \norm{\mathbf{W}_i}_F^2.
\end{equation}
Now, let us decompose $\mathbf{X}$ by its thin singular value decomposition such that $\mathbf{X}= \mathbf{U}\mathbf{\Sigma}\mathbf{V^\top}$, where $\mathbf{U}\in \mathbb{R}^{m\times r}$ , 
\(\mathbf{\Sigma}\in \mathbb{R}^{r\times r}\), 
\(\mathbf{V}\in \mathbb{R}^{n\times r}\), $\mathbf{U^\top U} = \mathbf{V^\top \mathbf{V}} = \mathbf{I}$ and $\Sigma = \operatorname{diag}(\sigma_{1}(\mathbf{X}), \cdots, \sigma_{r}(\mathbf{X}))$. We can configure a feasible point 
\(\hat{\mathbf{W}}_1,\hat{\mathbf{W}}_2,\dots,\hat{\mathbf{W}}_L\) by
\begin{equation}
\hat{\mathbf{W}}_L = \mathbf{U}\,\mathbf{\Sigma}^{\frac{1}{L}},\quad
\hat{\mathbf{W}}_{L-1} = \mathbf{\Sigma}^{\frac{1}{L}},\;\dots,\;
\hat{\mathbf{W}}_{2} = \mathbf{\Sigma}^{\frac{1}{L}},\quad
\hat{\mathbf{W}}_1= \mathbf{\Sigma}^{\frac{1}{L}}\,\mathbf{V}^\top.
\end{equation}
Therefore, 
\begin{align}
    \norm{\mathbf{X}}_{\mathcal{S}^{2/L}}^{2/L} &= \frac{1}{L} \sum_{i=1}^L\norm{\hat{\mathbf{W}}_i}_F^2 \\
    &\geq \min_{\substack{\mathbf{W}_1, \cdots,\mathbf{W}_L: \\
        \mathbf{W}_L \cdots \mathbf{W}_1 = \mathbf{X}}} \frac{1}{L} \sum_{i=1}^L \norm{\mathbf{W}_i}_F^2.
\end{align}
\end{proof}

\section{Stochastic and Doubly Stochastic Matrices} \label[appendix]{appendix:doubly}

\begin{definition}
    A \emph{permutation} of $[n]$ is a bijective function $\sigma: [n] \rightarrow [n]$. The identity permutation is defined as $\sigma(i) = i$ for all $i \in [n]$. 
    \label[definition]{defPerm}
\end{definition}
\begin{remark}
    For any $\mathbf{X} = [x_{ij}] \in \mathbb{F}^{n\times n}$, 
    \begin{equation}
        \operatorname{det}(\mathbf{X}) = \sum_{\sigma \in S_n} \operatorname{sgn}(\sigma)\prod_{i=1}^n x_{i\sigma(i)},
    \end{equation}
    where $\operatorname{sgn}$ is the sign function over all possible permutations from $[n]$ to $[n]$, which returns $+1$ if the permutation is even, and $-1$ otherwise. The sign of a permutation can be defined as $\operatorname{sgn}(\sigma) = (-1)^{N(\sigma)}$, where $N(\sigma)$ is the number of inversions in $\sigma$.
    \label[remark]{rem:definition_of_permutation}
\end{remark}

\begin{definition}
    A non-negative matrix $\mathbf{X} \in \mathbb{R}^{n \times n}$ satisfying the property $\mathbf{X}\mathbf{1} = \mathbf{1}$ is said to be a \emph{row stochastic matrix}. A \emph{column stochastic matrix} is simply the transpose of a row stochastic matrix. If it is simultaneously row and column stochastic, then it is called a \emph{doubly stochastic matrix}. 
\end{definition}

\begin{remark}
Suppose $\mathbf{X} \in \mathbb{R}^{n \times n}$ is a doubly stochastic matrix. If  $\mathbf{X}$ has only $n$ positive entries, then it is a permutation matrix. Furthermore, it has at most $n^2-n-1$ zero entries unless $\mathbf{X}$ is a permutation matrix.
\label[remark]{remark:iteration}
\end{remark}

\begin{definition}
    A non-negative matrix $\mathbf{X}$ is \emph{doubly substochastic} if all row and column sums are at most $1$, i.e., $\mathbf{X1} \leq \mathbf{1}$ and $\mathbf{1}^\top \mathbf{X} \leq \mathbf{1}^\top$. $N(\mathbf{X}) \in \mathbb{N}$ denotes the number of entries of the vectors $\mathbf{X1}$ and $\mathbf{X}^\top \mathbf{1}$ whose entries are less than $1$. 
    \label[definition]{defsubstochastic}
\end{definition}

\begin{lemma}[Theorem~8.7.1 \citet{horn2012matrix}]
    Suppose $\mathbf{X} \in \mathbb R^{n \times n}$ be a doubly stochastic matrix that is not the identity matrix. There exists a permutation $\sigma$ of $[n]$ that is not the identity permutation such that $x_{1\sigma(1)} x_{2\sigma(2)}\cdots x_{n\sigma{(n)}} > 0$. 
    \label[lemma]{theorem:sign}
\end{lemma}

\begin{proof}
For the sake of proof by contrapositive, suppose that every permutation $\sigma$ of $[n]$ that is not the identity permutation satisfies $x_{1\sigma(1)} x_{2\sigma(2)}\cdots x_{n\sigma{(n)}} = 0$ (see \Cref{defPerm}). Then, we can compute the characteristic polynomial of $\mathbf{X}$ as follows. 
\begin{align}
    \operatorname{det}(\lambda\mathbf{I}- \mathbf{X}) &= \prod_{i=1}^n (\lambda-x_{ii}) + \sum_{\sigma \neq \mathbf{I}} \left(\operatorname{sign}(\sigma) \prod_{i=1}^n -x_{i\sigma(i)}\right) \\ 
    &= \prod_{i=1}^n (\lambda-x_{ii}).
\end{align}
This implies that the diagonal entries of $\mathbf{X}$ are its eigenvalues. Note that at least one of the main diagonal entries of $\mathbf{X}$ must be $+1$ since $\mathbf{X}\mathbf{1} = \mathbf{1}$ by definition. This means that there exists a permutation matrix $\mathbf{P}$ such that 
\begin{equation}
    \mathbf{P}^\top \mathbf{X} \mathbf{P} = \begin{bmatrix}
        \mathbf{1} & \mathbf{0}\\
        \mathbf{0} & \mathbf{B}
    \end{bmatrix},
\end{equation}
where $\mathbf{B} \in \mathbb{R}^{n-1 \times n-1}$ is also doubly stochastic, and its main diagonal entries are obtained from the diagonal entries of $\mathbf{X}$. Applying the same procedure, the characteristic polynomial of $\mathbf{B}$ is 
\begin{equation}
    \operatorname{det}(\lambda \mathbf{I} -\mathbf{B}) = \prod_{i=1}^{n-1}(\lambda-b_{ii}).
\end{equation}
This means that the main diagonal entries of $\mathbf{B}$ are its eigenvalues. Since $\mathbf{B}$ is doubly stochastic, at least one of the main diagonal entries of $\mathbf{B}$ must be $+1$. Hence, at least two of the main diagonal entries of $\mathbf{X}$ is $+1$. Iterating at most $n-1$ steps in this way, we conclude that $\mathbf{X} = \mathbf{I}$.
\end{proof}
\begin{lemma}[Theorem~8.7.2, \citet{horn2012matrix}]
    A matrix $\mathbf{X} \in \mathbb{R}^{n \times n}$ is doubly stochastic \emph{if and only if} there are permutation matrices $\mathbf{P}_1, \mathbf{P}_2, \cdots, \mathbf{P}_N \in \mathbb{R}^{n\times n}$ and positive scalars $c_1,c_2, \cdots, c_N \in \mathbb R$ such that $\sum_{i}^N c_i = 1$ and 
    \begin{equation}
          \mathbf{X} = \sum_{i=1}^N c_i\mathbf{P}_i.  
    \end{equation}
    \label[lemma]{theorem:birkhoff}
\end{lemma}

\begin{proof}
    For any permutation matrices $\mathbf{P}_1, \mathbf{P}_2, \cdots, \mathbf{P}_N \in \mathbb{R}^{n\times n}$ and positive scalars $c_1,c_2, \cdots, c_N \in \mathbb R$ such that $\sum_{i}^N c_i = 1$, it is trivial to show that
    $\sum_{i=1}^N c_i\mathbf{P}_i$ is doubly stochastic. Therefore, we need to prove the converse implication of the proposition. 
    For the sake of proof by exhaustion, suppose that $\mathbf{X}$ is a permutation matrix. Then, the proof is complete. Now, suppose that $\mathbf{X}$ is not a permutation matrix. Then, by \Cref{theorem:sign}, there exists a non-identity permutation $\sigma$ of $[n]$ such that $x_{1\sigma(1)} x_{2\sigma(2)} \cdots x_{n\sigma({n})} > 0$. Let $\beta_1 := \min \left \{ x_{1\sigma(1)}, x_{2\sigma(2)}, \cdots, x_{n\sigma{(n)}}  \right \}$ and define the permutation matrix $\mathbf{P}_1$ such that it corresponds to the permutation $\sigma$.  If $\beta_1 = 1$, then $\mathbf{X}$ is a permutation matrix, which is a contradiction. Therefore, $\beta_1$ must be in $(0,1)$. Suppose $\mathbf{X}_1 := (1-\beta_1)^{-1} (\mathbf{X}- \beta_1 \mathbf{P}_1)$. Notice that 
    \begin{equation}
        \mathbf{P}_1 := [p_{ij}] \in \mathbb{R}^{n\times n} = \begin{cases}
1 
& j = \sigma(i),\\[1em]
0
& \text{otherwise}.
\end{cases}  
    \end{equation}
    Therefore, $\mathbf{X}_1$ has at least one more zero entry than $\mathbf{X}$. Additionally, the set of doubly stochastic matrices in $\mathbb{R}^{n \times n}$ is a convex set; hence, $\mathbf{X}_1$ is doubly stochastic. Furthermore, 
    \begin{equation}
        \mathbf{X}_1 := (1-\beta_1)^{-1} (\mathbf{X}- \beta_1 \mathbf{P}_1) \Leftrightarrow \mathbf{X} = (1-\beta_1)\mathbf{X}_1 + \beta_1\mathbf{P}_{1}.
    \end{equation}
    If $\mathbf{X}_1$ is a permutation matrix, the proof is complete. Otherwise, we can iterate this process at most $n^2 -n$ times using \Cref{remark:iteration}.
\end{proof}

\begin{corollary}
    A convex (concave) real-valued function over the set of doubly stochastic $n$ by $n$ matrices attains its maximum (minimum) at a permutation matrix.
    \label[corollary]{corollary:attainmax}
\end{corollary}

\begin{proof}
Suppose $g$ is a convex real-valued function over the set of $n$ by $n$ doubly stochastic matrices, and let $\mathbf{X}^*$ be a doubly stochastic matrix at which $g$ attains its maximum value. Note that the set of stochastic matrices is a compact set since its all entries lie in the closed interval $[0,1]$; therefore, the maximum is attainable. By using \Cref{theorem:birkhoff}, we can represent $\mathbf{X}^* = c_1 \mathbf{P}_1 + c_2 \mathbf{P}_2 + \cdots + c_N \mathbf{P}_N$ as a convex combination of permutation matrices. Let $j$ be an index such that $g(\mathbf{P}_j) = \max \{ g(\mathbf{P}_i): i \in [n]\}$. Then, 
\begin{align}
    g(\mathbf{X}^*) &= g(c_1 \mathbf{P}_1 + c_2 \mathbf{P}_2 + \cdots + c_N \mathbf{P}_N) \leq c_1 g(\mathbf{P}_1) + c_2 g(\mathbf{P}_2) + \cdots c_N g(\mathbf{P}_N) \\
    & \leq c_1 g(\mathbf{P}_j) +c_2g(\mathbf{P}_j) + \cdots + c_N g(\mathbf{P}_j)= g(\mathbf{P}_j).
\end{align}
Since $g$ attains its maximum at $\mathbf{X}^*$, we have $g(\mathbf{X}^*) = g(\mathbf{P}_j)$. The similar argument holds also for the case where $g$ is concave and attains its minimum at $\mathbf{X}^*$.
\end{proof}

\begin{lemma}[Lemma~8.7.5, \cite{horn2012matrix}]
Suppose $\mathbf{X} \in \mathbb R^{n \times n}$ is doubly substochastic. Then, there exists a doubly stochastic matrix $\mathbf{Y} \in \mathbb{R}^{n \times n}$ such that $\mathbf{X} \leq \mathbf{Y}$.
\label[lemma]{lemma:stochasticbound}
\end{lemma}

\begin{proof}
    Suppose $\mathbf{X} \in \mathbb{R}^{n \times n}$ is doubly substochastic. Let $N(\mathbf{X}) > 0$ (see \Cref{rem:definition_of_permutation}). This implies that there exist indices $i,j \in [n]$ such that 
    \begin{equation}
        \norm{\mathbf{X}\mathbf{e}_i}_1 < 1 \quad \text{and} \quad \norm{\mathbf{X}^\top\mathbf{e}_j}_1 < 1.
    \end{equation}
    Let us define the non-negative matrix $\Tilde{\mathbf{X}}$ such that
    \begin{equation}
    \norm{\left(\mathbf{X}+ \Tilde{\mathbf{X}} \right)\mathbf{e}_i}_1 = 1, \norm{\left(\mathbf{X}+ \Tilde{\mathbf{X}} \right)^\top \mathbf{e}_j}_1 = 1
    \end{equation}
    and 
    \begin{equation}
    \norm{\left(\mathbf{X}+ \Tilde{\mathbf{X}} \right)\mathbf{e}_k}_1 = \norm{\mathbf{X} \mathbf{e}_k}_1, \norm{\left(\mathbf{X}+ \Tilde{\mathbf{X}} \right)^\top\mathbf{e}_l}_1 = \norm{\mathbf{X}^\top \mathbf{e}_l}_1 \quad \forall (k,l): (k,l) \neq (i,j).
    \end{equation}
    Then, $\mathbf{X} + \Tilde{\mathbf{X}}$ is doubly substochastic, $\mathbf{X} \leq\mathbf{X} + \Tilde{\mathbf{X}}$ and $N(\mathbf{X} + \Tilde{\mathbf{X}}) < N(\mathbf{X})$. Note that any matrix $\mathbf{Z} \in \mathbb{R}^{n \times n}$ is doubly stochastic if and only if $N(\mathbf{Z}) = 0$. Therefore, we can iterate the same process $t$ times  until we reach a matrix $\mathbf{Y} = \mathbf{X} + \Tilde{\mathbf{X}}_t$ such that $N(\mathbf{Y}) = 0$.
\end{proof}

\section{Proof of von Neumann's Trace Inequality}
\label[appendix]{sec:proof_of_von_Neumann}
\begin{proof}
Let the singular value decompositions of $\mathbf{A}$ and $\mathbf{B}$ be $\mathbf{U}_1 \mathbf{\Sigma}_\mathbf A \mathbf{V}_1^\top$ and $\mathbf{U}_2\mathbf{\Sigma}_\mathbf{B} \mathbf{V}_2^\top$, respectively. Let $\mathbf{W} = \mathbf{V}_1^\top \mathbf{U}_2$ and $\mathbf{Y} = \mathbf{V}_2^\top \mathbf{U}_1$. Then 
\begin{align}
    \operatorname{tr}(\mathbf{AB}) &= \operatorname{tr}(\mathbf{U}_1 \mathbf{\Sigma}_\mathbf A \mathbf{V}_1^\top \mathbf{U}_2\mathbf{\Sigma}_\mathbf{B} \mathbf{V}_2^\top)   \\
    &= \operatorname{tr}(\mathbf{\Sigma}_\mathbf A \mathbf{W}\mathbf{\Sigma}_\mathbf{B}\mathbf{Y}) \\
    &= \sum_{i=1}^n \sum_{j =1}^n \sigma_i(\mathbf{A}) \sigma_j({\mathbf{B}}) w_{ij}y_{ji} \\
    &\leq \sum_{i=1}^n \sum_{j =1}^n \sigma_i(\mathbf{A}) \sigma_j({\mathbf{B}}) \abs{w_{ij}y_{ji}}.
\end{align}
Denote
\begin{equation}
    \mathbf{W} = \begin{bmatrix}
        \mathbf{w}_1 \\
        \mathbf{w}_2 \\
        \vdots\\
        \mathbf{w}_n
    \end{bmatrix} = \begin{bmatrix}
        \hat{\mathbf{w}}_1 & \hat{\mathbf{w}}_2 & \cdots & \hat{\mathbf{w}}_n
    \end{bmatrix},  \mathbf{Y} = \begin{bmatrix}
        \mathbf{y}_1 & \mathbf{y}_2 & \cdots & \mathbf{y}_n
    \end{bmatrix} = \begin{bmatrix}
        \hat{\mathbf{y}}_1 \\
        \hat{\mathbf{y}}_2 \\
        \vdots\\
        \hat{\mathbf{y}}_n,
    \end{bmatrix}
\end{equation}
Note that $\mathbf{W}$ and $\mathbf{Y}$ are orthogonal matrices; therefore, $\norm{\mathbf{w}_i}_2 = \norm{ \hat{\mathbf{w}}_i}_2 =  \norm{\mathbf{y}_i}_2 = \norm{ \hat{\mathbf{y}}_i}_2 = 1$ for all $i \in [n]$. Furthermore, $\abs{\mathbf{W} \odot \mathbf{Y}^\top} = [\abs{w_{ij}y_{ji}}] \in \mathbb{R}^{n\times n}$. 
The sum of the entries in any row of $[\abs{w_{ij}y_{ji}}]$ equals
\begin{equation}
    \sum_{j=1}^n \abs{w_{ij}y_{ji}} \leq \norm{\mathbf{w}_i}_2 \norm{\mathbf{y}_i}_2 = 1 \quad \forall i \in [n]
\end{equation}
by Hölder's inequality. Moreover, the sum of the entries in any column of $[\lvert w_{ij} y_{ji} \rvert]$ also satisfies
\begin{equation}
    \sum_{i=1}^n \lvert w_{ij} y_{ji} \rvert
    \leq \lVert \hat{\mathbf{w}}_j \rVert_2 \, \lVert \hat{\mathbf{y}}_j \rVert_2
    = 1 \qquad \forall\, j \in [n].
\end{equation}
Therefore, $[\abs{w_{ij}y_{ji}}]$ is doubly substochastic (see \Cref{defsubstochastic}). We know that there exists a doubly stochastic matrix $\mathbf{C}$ such that $[\abs{w_{ij}y_{ji}}] \leq \mathbf{C} = [c_{ij}] \in \mathbb{R}^{n \times n}$. Hence, 
\begin{equation}
    \operatorname{tr}(\mathbf{AB}) \leq \sum_{i=1}^{n} \sum_{j=1}^{n} \sigma_{i}(\mathbf{A}) \sigma_{j}(\mathbf{B}) c_{ij} =: f(C).
\end{equation}
Note that $f$ is a convex function over the set of doubly stochastic matrices. Therefore, it attains its maximum at a permutation matrix $\mathbf{P}= [p_{ij}] \in \mathbb{R}^{n \times n}$ (see \Cref{corollary:attainmax}). Denote by $\hat\pi :[n] \rightarrow [n]$ the permutation of $[n]$ corresponding to $\mathbf{P}$. Note that $p_{ij} = 1$ if and only if $\hat \pi(i) = j$. Therefore, 
\begin{equation}
    \operatorname{tr}(\mathbf{AB}) \leq \sum_{i=1}^{n} \sum_{j=1}^{n} \sigma_{i}(\mathbf{A}) \sigma_{j}(\mathbf{B}) p_{ij} = \sum_{i=1}^n \sigma_i(\mathbf{A})\sigma_{\hat \pi(i)}(\mathbf{B}).
\end{equation}
Define $S(\pi) =  \sum_{i=1}^n \sigma_i(\mathbf{A})\sigma_{\pi(i)}(\mathbf{B})$. For any permutation $\pi$ that is not the identity, there exists an inversion; that is, there exist indices $i,j \in [n]$ with $i<j$ such that $\pi(i) > \pi(j)$. We can obtain another permutation $\pi':[n] \rightarrow [n]$ from any permutation $\pi$ by simply swapping the values between $\pi(i)$ and $\pi(j)$, i.e., $\pi'(i) = \pi(j)$ and $\pi'(j) = \pi(i)$. Then
\begin{align}
    S(\pi) - S(\pi') &= \sigma_i(\mathbf{A})[\sigma_{\pi(i)}(\mathbf{B}) - \sigma_{\pi(j)}(\mathbf{B})] +  \sigma_j(\mathbf{A})[\sigma_{\pi(j)}(\mathbf{B}) - \sigma_{\pi(i)}(\mathbf{B})] \\
    &=\sigma_{\pi(i)}(\mathbf{B})[\sigma_{i}(\mathbf{A})-\sigma_{j}(\mathbf{A})] + \sigma_{\pi(j)}(\mathbf{B})[\sigma_{j}(\mathbf{A})-\sigma_{i}(\mathbf{A})] \\
    &= [\sigma_{i}(\mathbf{A})-\sigma_{j}(\mathbf{A})][\sigma_{\pi(i)}(\mathbf{B})-\sigma_{\pi(j)} (\mathbf{B})] \\
    &\leq 0.
\end{align}
This is equivalent to $S(\pi') - S(\pi) \geq 0$. This means that swapping an inversion does not decrease the sum. Since any permutation $\pi$ can be constructed from the identity permutation by a sequence of inversions,
\begin{equation}
    \sum_{i=1}^n \sigma_i(\mathbf{A})\sigma_{\hat \pi(i)}(\mathbf{B}) \leq \sum_{i=1}^n \sigma_i(\mathbf{A})\sigma_i(\mathbf{B}),
\end{equation}
which implies
\begin{equation}
    \operatorname{tr}(\mathbf{AB}) \leq  \sum_{i=1}^n \sigma_i(\mathbf{A})\sigma_i(\mathbf{B}),
\end{equation}
with equality if and only if $\mathbf{W} = \mathbf{Y} = \mathbf{I}$ if and only if $\mathbf{A}$ and $\mathbf{B}^\top$ have the same singular vectors.
\end{proof}
\noindent
Note that $\mathbf{A}$ and $\mathbf{B}$ do not need to be square matrices for this inequality to hold.
\begin{remark}
Let $\mathbf{A},\mathbf{B} \in \mathbb{R}^{m \times n}$ and $r = \min\{m,n\}$. Denote by  $\sigma_{1}(\mathbf{A}) \geq \cdots \geq \sigma_r(\mathbf{A})$ and $\sigma_{1}(\mathbf{B}) \geq \cdots \geq \sigma_r(\mathbf{B})$ the ordered singular values of $\mathbf{A}$ and $\mathbf{B}$, respectively. Then 
    \begin{equation}
        \operatorname{tr}(\mathbf{AB}^\top) \leq \sum_{i=1}^r \sigma_{i}(\mathbf{A})\sigma_i(\mathbf{B}).
    \end{equation}
    \label{corollary:VonNeumann}
\end{remark}

\begin{proof}
    Define $q = \max\{m,n\}$ and $\mathcal{A},\mathcal{B} \in \mathbb{R}^{q \times q}$ such that 
    \begin{equation}
        \mathcal{A} := \begin{bmatrix}
            \mathbf{A} & \mathbf{0}_{m \times (q-n)} \\
            \mathbf{0}_{(q-m) \times n} & \mathbf{0}_{(q-m) \times (q-n)}
        \end{bmatrix},
        \mathcal{B} := \begin{bmatrix}
            \mathbf{B}^\top & \mathbf{0}_{n \times (q-m)} \\
            \mathbf{0}_{(q-n)\times m} & \mathbf{0}_{(q-n) \times (q-m)}
        \end{bmatrix},
    \end{equation}
    where 
    \begin{equation}
        \mathcal{AB} = \begin{bmatrix}
            \mathbf{AB}^\top & \mathbf{0}_{m \times (q-m)} \\
            \mathbf{0}_{(q-m) \times m} & \mathbf{0}_{(q-m) \times (q-m)}.
        \end{bmatrix}
    \end{equation}
    Therefore, 
    \begin{equation}
        \operatorname{tr}(\mathbf{AB}^\top) = \operatorname{tr}(\mathcal{AB}) \leq \sum_{i=1}^q\sigma_i(\mathcal{A})\sigma_i(\mathcal{B}) = \sum_{i=1}^r\sigma_i(\mathbf{A})\sigma_i(\mathbf{B}),
    \end{equation}
    with equality if and only if $\mathbf{A}$ and $\mathbf{B}$ have the same singular vectors.
\end{proof}

\end{document}